\documentclass[]{interact}

\usepackage{epstopdf}
\usepackage[caption=false]{subfig}

\usepackage[numbers,sort&compress]{natbib}
\bibpunct[, ]{[}{]}{,}{n}{,}{,}
\renewcommand\bibfont{\fontsize{10}{12}\selectfont}

\theoremstyle{plain}
\newtheorem{theorem}{Theorem}[section]
\newtheorem{lemma}[theorem]{Lemma}

\newtheorem{proposition}[theorem]{Proposition}

\theoremstyle{definition}
\newtheorem{definition}[theorem]{Definition}

\theoremstyle{remark}
\newtheorem{remark}{Remark}

\usepackage{xcolor}
\usepackage{amsmath}
\usepackage[title]{appendix}
\usepackage{algorithm}
\usepackage{algorithmic}
\usepackage{multirow}
\usepackage{booktabs}
\usepackage{longtable}
\usepackage{url}

\begin{document}


\title{Learning Graph Laplacian with MCP}

\author{
\name{
Yangjing Zhang\textsuperscript{a}
\thanks{Email: yangjing.zhang@amss.ac.cn}
Kim-Chuan Toh\textsuperscript{b}
\thanks{Email: mattohkc@nus.edu.sg}
and Defeng Sun\textsuperscript{c}
\thanks{Email: defeng.sun@polyu.edu.hk}
}
\affil{\textsuperscript{a}Institute of Applied Mathematics, Academy of Mathematics and Systems Science, Chinese Academy of Sciences, Beijing, People's Republic of China; \textsuperscript{b}Department of Mathematics, and Institute of Operations Research and Analytics, National University of Singapore, Singapore;
\textsuperscript{c}Department of Applied Mathematics, The Hong Kong Polytechnic University, Hong Kong
}}

\maketitle

\begin{abstract}
We consider the problem of learning a graph under the Laplacian constraint with a non-convex penalty: minimax concave penalty (MCP). For solving the MCP penalized graphical model, we design an inexact proximal difference-of-convex algorithm (DCA) and prove its convergence to critical points. We note that each subproblem of the proximal DCA enjoys the nice property that the objective function in its dual problem is continuously differentiable with a semismooth gradient. Therefore, we apply  an efficient semismooth Newton method to subproblems of the proximal DCA. Numerical experiments on various synthetic and real data sets demonstrate the effectiveness of the non-convex penalty MCP in promoting sparsity. Compared with the existing state-of-the-art method, our method is demonstrated to be  more efficient and reliable  for learning graph Laplacian with MCP.
\end{abstract}



\section{Introduction}
In modern multivariate data analysis, one of the most important problems is the estimation of the precision matrix (or the inverse covariance matrix) of a multivariate distribution via an undirected graphical model. A Gaussian graphical model for a Gaussian random vector $\Delta \sim \mathcal{N}_n(\mu,\Sigma)$ is represented by a graph $\mathcal{G} = (\mathcal{V},\mathcal{E})$, where $\mathcal{V}$ is a collection of $n$ vertices corresponding to the $n$ random variables (features), and an edge $(i,j)$ is absent, i.e., $(i,j) \notin \mathcal{E}$,  if and only if the $i$-th and $j$-th random variables are conditionally independent of each other, given all other variables. The conditional independence is further equivalent to having the $(i,j)$-th entry of the precision matrix $(\Sigma^{-1})_{ij}$ being zero \cite{lauritzen1996graphical}.
Thus, finding the graph structure of a Gaussian graphical model is equivalent to the identification of zeros in the corresponding precision matrix.

Let $\mathbb{S}^n_+$ ($\mathbb{S}^n_{++}$)  denote the cone of positive semidefinite (definite) matrices in the space of $n\times n$ real symmetric matrices $\mathbb{S}^n$.
Given a Gaussian random vector $\Delta \sim \mathcal{N}_n(\mu,\Sigma)$ and its sample covariance matrix $S \in \mathbb{S}^n_+$, a notable way of learning a precision matrix from the data matrix $S$ is via the following $\ell_1$-norm penalized maximum likelihood approach \cite{yuan2007model,banerjee2008model,friedman2008sparse,rothman2008sparse}
\begin{equation}\label{model-glasso}
\begin{array}{l}
\min\limits_{\Theta\in \mathbb{S}^n_{++}}~-\log\det \,\Theta+\langle S,\Theta \rangle +
{\lambda\|\Theta\|_{1,{\rm off}}},
\end{array}
\end{equation}
where {$\|\Theta\|_{1,{\rm off}} = \sum_{i\neq j} |\Theta_{ij}|$} and $\lambda$ is a non-negative penalty parameter.
The solutions to \eqref{model-glasso} are only constrained to be positive definite, and both positive and negative edge weights are allowed in the estimated graph. However, one may further require all edge weights to be non-negative. A negative edge weight implies a negative partial correlation between the two connected random variables, which might be difficult to interpret in some applications. For certain types of data, one feature is likely to be predicted by non-negative linear combinations of other features. Under such application settings, the extra non-negative constraints on the edge weights can provide a more accurate estimation of the graph than \eqref{model-glasso} and thus one prefers to estimate a graph Laplacian matrix from the data. More broadly, graph Laplacian matrices are desirable for a large majority of studies and applications, for example, spectral graph theory \cite{chung1997spectral}, clustering and partition problems \cite{simon1991partitioning,shi1997normalized,ng2002spectral},  dimensionality reduction and data representation \cite{belkin2003laplacian}, and graph signal processing \cite{shuman2013emerging}. Therefore, it is essential to learn graph Laplacian matrices from data.
{In fact, many researchers have recently considered Gaussian graphical models that is multivariate totally positive of order 2 (MTP$_2$), namely, all partial correlations are required to be non-negative. A Gaussian graphical model is said to be MTP$_2$ if the $(i,j)$-th entry of the precision matrix $(\Sigma^{-1})_{ij}$ is non-positive for all $i\neq j$. We refer the readers to \cite{fallat2017total,lauritzen2019maximum,wang2020learning,slawski2015estimation,tugnait2021high,tugnait2021sparse,deng2020fast,egilmez2017graph} for properties of MTP$_2$ Gaussian graphical models and efficient algorithms for estimating the MTP$_2$ Gaussian graphical models.}

We start by giving the definition of a graph Laplacian matrix formally. For an undirected weighted graph $\mathcal{G} = (\mathcal{V},\mathcal{E})$ with $\mathcal{V}$ being the set of vertices and $\mathcal{E}$ the set of edges, the weight matrix of $\mathcal{G}$ is defined as $W\in\mathbb{S}^n$ where $W_{ij} = w_{(ij)}\in\mathbb{R}_+$ is the weight of the edge $(i,j)\in\mathcal{E}$ and $W_{ij} = 0$ if $(i,j)\notin\mathcal{E}$. Therefore, the weight matrix $W$ consists of non-negative off-diagonal entries and zero diagonal entries. The graph Laplacian matrix, also known as the combinatorial graph Laplacian matrix, is defined as
$$ L = D-W\in\mathbb{S}^n,$$
where $D$ is the diagonal  matrix such that $D_{ii} = \sum_{j=1}^{n} W_{ij}$. The connectivity  (adjacency)  matrix $A\in\mathbb{S}^n$ is defined as the
sparsity pattern of $W$, i.e., $A_{ij} = 1$ if $W_{ij} > 0$, and $A_{ij} = 0$ if $W_{ij} = 0$. The set of graph Laplacian matrices then consists of matrices with non-positive off-diagonal entries and zero row-sum
\begin{equation}\label{set-lap}
\mathcal{L}=\{\Theta\in\mathbb{S}^n\,|\,\Theta{\bf 1}={\bf 0},\Theta_{ij}\leq 0 \mbox{ for } i\neq j\}.
\end{equation}
Here, ${\bf 1}$ (${\bf 0}$) denotes the vector of all ones (zeros). If the graph connectivity matrix $A$ is known a priori, then the constrained set of  graph Laplacian matrices is
\begin{equation}\label{set-lap2}
\mathcal{L}(A) = \left\{\Theta\in\mathbb{S}^n\,|\,\Theta{\bf 1}={\bf 0},
\begin{array}{l}
\Theta_{ij}\leq 0 \mbox{ if } A_{ij} = 1\\
\Theta_{ij}= 0 \mbox{ if } A_{ij} = 0
\end{array}
\mbox{ for } i\neq j
\right\}.
\end{equation}
As we know, a Laplacian matrix is diagonally dominant and positive semidefinite, and it has a zero eigenvalue with the associated eigenvector ${\bf 1}$. If the graph is connected, then the null space of the  Laplacian matrix is one-dimensional and spanned by the vector ${\bf 1}$ \cite{chung1997spectral}.

Starting from the earlier work of imposing the Laplacian structure \eqref{set-lap} on the estimation of precision matrix in \cite{lake2010discovering},
this line of research has seen a recent surge of interest in
\cite{hu2013graph,hassan2016topology,kalofolias2016learn,dong2016learning,egilmez2017graph,zhao2019optimization,kumar2020unified}.
To handle the singularity issue of a Laplacian matrix in the calculation of the log-determinant term,  Lake and Tenenbaum \cite{lake2010discovering} considered a regularized graph Laplacian matrix $\Theta + \alpha I$ (hence, full rank) by adding  a positive scalar $\alpha$ to the diagonal entries of the graph Laplacian matrix $\Theta$. More recently, by considering connected graphs,  Egilmez et al. \cite{egilmez2017graph} and Hassan-Moghaddam et al. \cite{hassan2016topology} used a modified version of \eqref{model-glasso} by adding the constant matrix $J = (1/n){\bf 1}{\bf 1}^T$ to the graph Laplacian matrix to compensate for the null space spanned by the vector ${\bf 1}$. Moreover, Egilmez et al. \cite{egilmez2017graph} incorporated the connectivity matrix $A$ into the model to exploit any prior structural information about the graph. Given the connectivity matrix $A$ and a data matrix $S$ (typically a sample covariance matrix),  Egilmez et al. \cite{egilmez2017graph} proposed the following $\ell_1$-norm penalized combinatorial graph Laplacian \eqref{cgl-l1} estimation model
\begin{equation}\label{cgl-l1}
\begin{array}{cl}
\min & \{-\log\det\, (\Theta + J) + \langle S,\Theta \rangle + {\lambda\|\Theta\|_{1,{\rm off}}}\,|\,\Theta\in\mathcal{L}(A) \}.
\end{array}\tag{CGL-L1}
\end{equation}
When the prior knowledge of  the structural information $A$ is not available, especially for real data sets, one can take the fully connected matrix as the connectivity matrix, i.e., $A={\bf 1}{\bf 1}^T - I$, with $I$ being an identity matrix. In this case, the model \eqref{cgl-l1} involves estimating both the graph structure and graph edge weights.
The model \eqref{cgl-l1} is a natural extension of the classical  model \eqref{model-glasso} due to the equality
$\log\det\, (\Theta + J) = \log {\rm pdet}\, \Theta$
for any Laplacian matrix $\Theta$ of a connected graph. Here ${\rm pdet}\,(\cdot)$ denotes the pseudo-determinant of a square matrix, i.e.,  the product of all non-zero eigenvalues of the matrix.

The model \eqref{cgl-l1} naturally extends  the classical model \eqref{model-glasso} to incorporate the graph Laplacian constraint. {However,  the model \eqref{cgl-l1} has been proved to have the drawback that the $\ell_1$ penalty may lose its power in promoting sparsity in the estimated graph \cite{ying2020nonconvex}.
In fact, an intuitive explanation of this phenomena is by the zero row-sum constraint} of a valid Laplacian matrix $\Theta\in\mathcal{L}(A)$, which satisfies
$$ {\lambda\|\Theta\|_{1,{\rm off}}} = -\lambda \sum_{i\neq j}\Theta_{ij} = \lambda \sum_i \Theta_{ii} = \langle \lambda I,\Theta\rangle. $$
Thus the $\ell_1$ penalty term in \eqref{cgl-l1} simply penalizes the diagonal entries of $\Theta$ but not the individual  entries $\Theta_{ij}$. Hence adjusting $\lambda$ may not affect the sparsity level of the solution of \eqref{cgl-l1}.

Motivated by the observation above, the work \cite{ying2020nonconvex} proposed to apply a non-convex penalty function to promote sparsity. Non-convex penalties can generally reduce estimation bias, and they have been applied in sparse precision matrix estimation in \cite{lam2009sparsistency,shen2012likelihood}.
In this paper, we aim to propose an efficient algorithm to learn a graph Laplacian matrix as the precision matrix from the constrained maximum likelihood estimation with the minimax concave penalty (MCP) \cite{zhang2010nearly}
\begin{equation}\label{cgl-mcp}
\begin{array}{cl}
\min & \{-\log\det\, (\Theta + J) + \langle S,\Theta \rangle + P(\Theta) \,|\,\Theta\in\mathcal{L}(A) \},
\end{array}\tag{CGL-MCP}
\end{equation}
where the MCP function $P$ is given as follows
(we omit its dependence  on the parameters $\gamma$ and $\lambda$ for brevity):
\begin{equation}\label{fct-MCP}
\begin{array}{cl}
P(\Theta) &= \displaystyle \sum_{i\neq  j} p_{\gamma}(\Theta_{ij};\lambda),\,\gamma > 1, \mbox{ for } \Theta\in\mathbb{S}^n,\\
p_{\gamma}(x;\lambda)&=
\begin{cases}
  \lambda |x| - \frac{x^2}{2\gamma}, & \mbox{if } |x|\leq \gamma\lambda, \\
  \frac{1}{2}\gamma \lambda^2, & \mbox{if } |x| >\gamma\lambda,
\end{cases}
\mbox{ for } x\in\mathbb{R},\,\lambda>0.
\end{array}
\end{equation}
It is known that the MCP can be expressed as the difference of two convex (d.c.) functions \cite[Section~6.2]{ahn2017difference}. Therefore, we can design a d.c. algorithm (DCA) for solving \eqref{cgl-mcp} by using the d.c. property of the MCP function.
DCA is an important tool for solving d.c. programs, and numerous research has been conducted on this topic; see \cite{tao1997convex,vo2015learning,pang2017computing,le2018dc}, to name only a few.
Specifically in this paper, we aim to develop an inexact  proximal DCA for solving \eqref{cgl-mcp}.
We solve approximately a sequence of convex minimization subproblems by finding an affine minorization of the second d.c. component. Allowing inexactness for solving the subproblem is crucial since computing the exact solution of the subproblem is generally impossible. The novel step of our algorithm is to introduce a proximal term for each subproblem. With the proximal terms, we can obtain a convex and continuously differentiable  dual problem of the subproblem. Due to the nice property of the dual formulation, we can apply a highly  efficient semismooth Newton (SSN) method  \cite{kummer1988newton,qi1993nonsmooth,sun2002semismooth} for solving the dual of the subproblem.
For completeness, we also prove the convergence of the inexact proximal DCA to a critical point.

{\it Summary of contributions:}
1) We propose an inexact proximal DCA for  solving the  model \eqref{cgl-mcp}, and prove its convergence to a critical point.
2) As the subroutine for solving the proximal DCA subproblems, we design a semismooth Newton method for solving the dual form of the subproblems.
3) The effectiveness of the proposed model \eqref{cgl-mcp}  and the efficiency of the inexact proximal DCA for solving \eqref{cgl-mcp} are comprehensively demonstrated via various experiments on both synthetic and real data sets.
4) More generally, both the model and algorithm can be extended directly to other non-convex penalties, such as the smoothly clipped absolute deviation (SCAD) function \cite{fan2001variable}.

{\it Outline:} The remainder of the paper is organized as follows. The  inexact proximal DCA for solving \eqref{cgl-mcp}, together with its convergence property, is given in Section~\ref{sec-mm}. Section~\ref{sec-ssn} presents the semismooth Newton  method for solving  the subproblems of the proximal DCA.
Numerical experiments are presented in Section~\ref{sec-exp}. Finally, Section~\ref{sec-con} concludes the paper.

\section{Inexact Proximal DCA for \eqref{cgl-mcp}}\label{sec-mm}
We propose an inexact proximal DCA for solving \eqref{cgl-mcp} in this section and prove its convergence to a critical point.

\subsection{Reformulation of \eqref{cgl-mcp}}
In order to take full advantage of the sparsity of the graph via the connectivity matrix, we reformulate the model  \eqref{cgl-mcp} with the  vector of edge weights as the decision variable.  This formulation eliminates the extra computation incurred by non-existence edges ($A_{ij}=0$).

Given an undirected  graph $\mathcal{G}$ and its connectivity matrix $A$, we know that the edge set can be characterized by the connectivity matrix: $\mathcal{E}=\{(i,j)\,|\,A_{ij}=1,\,i<j\}$.
Let $\mathbb{R}^{|\mathcal{E}|}$ be the vector space such that for any $w\in \mathbb{R}^{|\mathcal{E}|}$, the components of $w$ are indexed by the elements of $\mathcal{E}$, i.e., $[w_{(ij)}]_{(i,j)\in\mathcal{E}}$.
We also let $B\in\mathbb{R}^{n\times |\mathcal{E}|}$ be the node-arc incidence matrix such that the $(ij)$-th column is given by $B_{(ij)} = e_i - e_j$, where $e_i,\,i=1,\dots,n$ are the standard unit vectors in $\mathbb{R}^n$. {It is well known that the Laplacian matrix of the graph $\mathcal{G}$ can also be given by $B{\rm Diag}(w)B^T$ with $w$ being the vector of edge weights.}
Thus, we define a linear map $\mathcal{A}^*:\,\mathbb{R}^{|\mathcal{E}|} \to \mathbb{S}^n$:
$$
\mathcal{A}^*w = B{\rm Diag}(w)B^T,\,\forall\,w\in\mathbb{R}^{|\mathcal{E}|}.
$$
We can see that $\mathcal{A}^*w\in \mathcal{L}(A)$ if $w\in\mathbb{R}^{|\mathcal{E}|}$ is a non-negative weight vector by noting that
\begin{equation*}
(\mathcal{A}^*w)_{ij}=\left\{
\begin{array}{ll}
 \sum\limits_{(i,k)\in\mathcal{E}} w_{(ik)} + \sum\limits_{ (k,i)\in\mathcal{E}}w_{(ki)}, & \mbox{ if } i=j, \\
 -w_{(ij)}, & \mbox{ if } (i,j) \in \mathcal{E} \\
 -w_{(ji)}, & \mbox{ if } (j,i) \in \mathcal{E}, \\
 0,         & \mbox{ otherwise}.
\end{array}\right.
\end{equation*}
It is easy to obtain from the definition that the adjoint map $\mathcal{A}:\,\mathbb{S}^n  \rightarrow \mathbb{R}^{|\mathcal{E}|}$ is given by $\mathcal{A}X = {\rm diag}(B^TXB),\,\,\forall\, X \in\mathbb{S}^n$. With the relationship between the Laplacian matrix and the weight vector, the model \eqref{cgl-mcp} can be reformulated as follows:
\begin{equation}\label{cgl-mcp2}
  \begin{array}{cl}
  \min\limits_{w} & -\log\det\, ( \mathcal{A}^*w+ J) + \langle S, \mathcal{A}^*w \rangle + P( \mathcal{A}^*w)\\[6pt]
  {\rm s.t.} & w \in \mathbb{R}^{|\mathcal{E}|}_+.
  \end{array}
\end{equation}
We note that the constraint $\Theta\in\mathcal{L}(A)$ in the original  model \eqref{cgl-mcp} is reformulated as a linear constraint and a non-negative constraint, namely, $\Theta = \mathcal{A}^*w$ and $w \in \mathbb{R}^{|\mathcal{E}|}_+$.

\subsection{Inexact Proximal DCA for Solving \eqref{cgl-mcp2}}
The function $p_{\gamma}(\cdot;\lambda)$ defined in \eqref{fct-MCP} can be expressed as the difference of two convex  functions \cite[Section~6.2]{ahn2017difference}  as follows:
$$ p_{\gamma}(x;\lambda) = \lambda |x| - h_{\gamma}(x;\lambda), $$
where $\gamma$ and $\lambda$ are given positive parameters, and
\begin{equation*}\label{fct-h}
h_{\gamma}(x;\lambda)=
\begin{cases}
 \frac{x^2}{2\gamma}, & \mbox{if } |x|\leq \gamma\lambda, \\
 \lambda|x| - \frac{1}{2}\gamma \lambda^2, & \mbox{if } |x| >\gamma\lambda,
\end{cases}
\mbox{ for } x\in\mathbb{R},\,\lambda>0.
\end{equation*}
Its gradient is given by
$ \nabla_x h_{\gamma}(x;\lambda) = \min\big(\frac{|x|}{\gamma},\lambda\big){\rm sign}(x)$. ${\rm sign}(\cdot)$ is defined such that ${\rm sign}(t)=1$ if $t>0$, ${\rm sign}(t) = 0$ if $t=0$, and ${\rm sign}(t)=-1$ if $t<0$.
Therefore, the MCP function $P$ in \eqref{fct-MCP} can also be written as the difference of two convex functions
\begin{equation}\label{dc-P}
\begin{array}{cl}
P(\Theta) &= \lambda \|\Theta\|_{ 1,{\rm off}} - h(\Theta),\\[3mm]
h(\Theta) &= \displaystyle\sum_{i\neq j} h_{\gamma}(\Theta_{ij};\lambda), \mbox{ for } \Theta\in\mathbb{S}^n.
\end{array}
\end{equation}
Note that  $h:\mathbb{S}^n\to\mathbb{R}$ is a convex and continuously differentiable function with its gradient given by
\begin{equation}\label{grad-h}
\left[\nabla h(\Theta) \right]_{ij} =
\begin{cases}
\min\Big(\frac{|\Theta_{ij}|}{\gamma},\lambda\Big){\rm sign}(\Theta_{ij}), & \mbox{if } i\neq j, \\
0, & \mbox{if } i=j.
\end{cases}
\end{equation}
With the above preparation, we can now propose an inexact proximal DCA (Algorithm~\ref{alg-mm}) for solving \eqref{cgl-mcp2}. We adopt an approximate solution of the \eqref{cgl-l1} as the initial point. Note that the \eqref{cgl-l1} can be solved by the alternating direction method of multipliers; see Appendix~\ref{app:B} for its implementation. It is important to note that Step~0 is not a necessary component of Algorithm~\ref{alg-mm}.
Instead, one can simply initialize with an arbitrary value $w^0$ and proceed with Steps~1 and 2. However, we include Step~0 in order to establish connections between the non-convex problem \eqref{cgl-mcp} and the convex problem \eqref{cgl-l1}.

We motivate the choice of the initial point in Step 0 with the following observation. If we choose a simple and natural initialization of  $w^0=0\in \mathbb{R}^{|\mathcal{E}|}$ and perform Step 1 at $k=0$, the resulting subproblem \eqref{subprob} takes an approximate form of \eqref{cgl-l1} with proximal terms:
\begin{equation*}
\underset{w}{\min} ~~ \Big\{-\log\det\, ( \mathcal{A}^*w + J) + \langle S + \lambda I, \mathcal{A}^*w \rangle
+ \langle \delta^0,w\rangle + \frac{\sigma_0}{2}\|w \|^2 + \frac{\sigma_0}{2}\|\mathcal{A}^*w \|^2  \,\,\Big|\,\,w \in \mathbb{R}^{|\mathcal{E}|}_+ \Big\},
\end{equation*}
where $\delta^0$ is the error vector and $\sigma_0>0$ is a given constant. This problem is equivalent to  \eqref{cgl-l1} with proximal terms. Therefore, we approximately solve \eqref{cgl-l1} to obtain an initial point.
Lastly, for the inexactness condition \eqref{stop-cond}, we make the convention that $\frac{\sigma_k}{4}\|w^{k+1} - w^k\| +
\frac{\sigma_k \| \mathcal{A}^* w^{k+1}- \mathcal{A}^*w^k\|^2 }{2\|w^{k+1} - w^k\|}=0$ when $w^{k+1} = w^k$. It is natural because this term goes to zero as $\|w^{k+1} - w^k\|\to 0$. If $w^{k+1} = w^k$ and $\delta^k=0$, it follows from the optimality condition of \eqref{subprob} that $w^k$ is a critical point of \eqref{cgl-mcp3}.

\begin{algorithm}[h]
\caption{Inexact proximal DCA  for solving \eqref{cgl-mcp}}\label{alg-mm}

Given $\lambda>0$, $\gamma > 1$, and $\sigma_0 > 0$.

\begin{description}
\item[Step 0.]
Solve \eqref{cgl-l1} approximately
\begin{equation*}
\begin{array}{l}
w^0 \approx \arg\underset{w}{\min} ~~\Big\{-\log\det\, ( \mathcal{A}^*w + J) + \langle S + \lambda I, \mathcal{A}^*w \rangle \,\,\Big|\,\,w \in \mathbb{R}^{|\mathcal{E}|}_+ \Big\}.
\end{array}
\end{equation*}
Go to {\bf Step 1}.

\vspace{2mm}
\item[Step 1.]
Solve the following problem with $G^k:=S+\lambda I - \nabla h(\mathcal{A}^*w^k)$
\begin{equation}\label{subprob}
\begin{array}{rl}
w^{k+1}= \arg\underset{w}{\min} & \Big\{
-\log\det\, (\mathcal{A}^*w + J) + \langle G^k,\mathcal{A}^*w \rangle + \langle \delta^k,w\rangle \\[6pt]
 & + \frac{\sigma_k}{2}\|w - w^k\|^2 + \frac{\sigma_k}{2}\|\mathcal{A}^*w - \mathcal{A}^* w^k\|^2 \,\,\Big|\,\, w \in \mathbb{R}^{|\mathcal{E}|}_+ \Big\},
\end{array}
\end{equation}
such that the error vector $\delta^k$ satisfies the stopping condition
\begin{equation}\label{stop-cond}
\|\delta^k\| \leq \frac{\sigma_k}{4}\|w^{k+1} - w^k\| +
\frac{\sigma_k \| \mathcal{A}^* w^{k+1}- \mathcal{A}^*w^k\|^2 }{2\|w^{k+1} - w^k\|}.
\end{equation}

\item[Step 2.] If a prescribed stopping criterion is satisfied, terminate; otherwise update $\sigma_{k+1}  \leftarrow \rho_k \sigma_k$ with $\rho_k\in{(0,1]}$ and return to {\bf Step 1} with $k \leftarrow k+1$.

\end{description}
\end{algorithm}

\subsection{Convergence Analysis of the Inexact Proximal DCA}
In this section, we prove that the limit point of the sequence generated by our inexact proximal DCA (Algorithm~\ref{alg-mm}) will converge to a critical point of the non-convex problem \eqref{cgl-mcp2}. We also establish the sequential convergence and convergence rate results under additional assumptions.
Although there exist convergence results of  d.c. algorithms, we delicately designed proximal terms and stopping conditions in our framework which may be different from existing works. Therefore we give a complete proof of convergence here.
We can rewrite our problem in the form
\begin{equation}\label{cgl-mcp3}
\inf  \{f(w):=g(w) - h(\mathcal{A}^*w)\},
\end{equation}
where $h$ is defined in \eqref{dc-P} and
\begin{equation*}
g(w) := -\log\det\, (\mathcal{A}^*w + J) + \langle S+\lambda I,\mathcal{A}^*w \rangle + \delta (w\,|\,\mathbb{R}^{|\mathcal{E}|}_+ ),\,w\in\mathbb{R}^{|\mathcal{E}|}.
\end{equation*}
Here $\delta(\cdot\,|\,C)$ denotes the indicator function of any convex set $C$.
The function $g$ is lower semi-continuous (l.s.c) and convex, and the function $h$ is continuously differentiable and convex.
A point $w$ is said to be a critical point of \eqref{cgl-mcp3} if
$$
\mathcal{A}\nabla h(\mathcal{A}^*w) \in \partial g(w),
$$
where  $\partial g(w)$ is the subdifferential of the convex function $g$ at $w$.
At each iteration $k$, we define the following function which majorizes $f$
\begin{equation}\label{major-f}
f^k(w) := g(w) - h(\mathcal{A}^*w^k) - \langle \nabla h(\mathcal{A}^*w^k),\mathcal{A}^*w - \mathcal{A}^*w^k\rangle+ \frac{\sigma_k}{2}\|w - w^k\|^2 + \frac{\sigma_k}{2}\|\mathcal{A}^*w - \mathcal{A}^* w^k\|^2.
\end{equation}
One can observe that the majorized subproblem \eqref{subprob}  can be expressed equivalently as
\begin{equation}\label{subprob-eq}
w^{k+1}=\arg\min_w ~~\Big\{f^k(w) + \langle \delta^k,w\rangle \,\,\Big|\,\,w \in \mathbb{R}^{|\mathcal{E}|}_+ \Big\}.
\end{equation}
The following lemma gives the descent property of the objective function $f$.
\begin{lemma}\label{lemma2.1}
Let $w^{k+1}$ be an approximate solution to problem \eqref{subprob} such that {the stopping condition} \eqref{stop-cond} holds. Then we have that
\begin{equation*}
f(w^{k+1}) \leq f(w^k) - \frac{\sigma_k}{4}\|w^{k+1} - w^k\|^2.
\end{equation*}
If the sequence  $\{f(w^k)\}$ is bounded below, then  the sequence $\{f(w^k)\}$  converges to a finite number.
\end{lemma}
\begin{proof}
By the update rule of the algorithm, we have
\begin{equation*}
\begin{array}{cl}
f(w^k) = f^k(w^k) &\geq f^k(w^{k+1}) + \langle \delta^k,w^{k+1} - w^k\rangle ~~~~\mbox{by \eqref{subprob-eq}}\\[6pt]
&\geq f^k(w^{k+1}) - \frac{\sigma_k}{4}\|w^{k+1} - w^k\|^2 -  \frac{\sigma_k}{2}\|\mathcal{A}^*w^{k+1} - \mathcal{A}^*w^k\|^2 ~~~~\mbox{by \eqref{stop-cond}}\\[6pt]
&\geq  f(w^{k+1}) + \frac{\sigma_k}{4}\|w^{k+1} - w^k\|^2 ~~~~\mbox{(by \eqref{major-f} and the convexity of $h$)}.
\end{array}
\end{equation*}
\end{proof}
The next theorem states the convergence of the inexact proximal DCA to a critical point.
We note that when the function $h$ is continuously differentiable, the notion of criticality and d-stationarity coincide, see, e.g.,
\cite[Section~3.2]{pang2017computing}.

\begin{theorem}\label{thm-subcon}
Suppose that the d.c. function $f$ is bounded below. Assume that the sequence $\{\sigma_k\}$ is convergent to $\sigma_{\infty}>0$. Let $\{w^k\}$ be the sequence generated by Algorithm~\ref{alg-mm}. Then every limit point of $\{w^k\}$, if exists, is a critical point of \eqref{cgl-mcp3}.
\end{theorem}
\begin{proof}
Since $f$ is bounded below, it follows from Lemma~\ref{lemma2.1} that $\lim_{k\to\infty} f(w^k)$ exists and
$$
\lim_{k\to\infty} [ f(w^k) - f(w^{k+1}) ] = \lim_{k\to\infty} \|w^{k+1} - w^k\| = 0.
$$
Let $\{w^k\}_{k\in \kappa}$ be a sequence converging to a limit $w^{\infty}$. Then $\{w^{k+1}\}_{k\in \kappa}$ also converges to $w^{\infty}$.
By \eqref{subprob-eq}, we have for any $w$
$$
f^k(w) \geq f^k(w^{k+1}) + \langle \delta^k,w^{k+1} - w\rangle
\geq f^k(w^{k+1}) - \| \delta^k\| \|w^{k+1} - w\|.
$$
Using the definition \eqref{major-f} and taking limit $k(\in\kappa) \to\infty$ yield that for any $w$
$$
g(w) \geq g(w^{\infty}) + \langle \nabla h(\mathcal{A}^*w^{\infty}),\mathcal{A}^*w-\mathcal{A}^*w^{\infty}\rangle
 - \frac{\sigma_{\infty}}{2}\|w - w^{\infty}\|^2 -  \frac{\sigma_{\infty}}{2}\|\mathcal{A}^*w - \mathcal{A}^*w^{\infty}\|^2,
$$
where $\sigma_{\infty} = \lim_{k\to\infty} \sigma_k$. This is equivalent to saying that
$$
\begin{array}{cl}
w^{\infty}=\arg\min_w &\{g(w) - \langle \nabla h(\mathcal{A}^*w^{\infty}),\mathcal{A}^*w-\mathcal{A}^*w^{\infty}\rangle\\[6pt]
&+ \frac{\sigma_{\infty}}{2}\|w - w^{\infty}\|^2 +  \frac{\sigma_{\infty}}{2}\|\mathcal{A}^*w - \mathcal{A}^*w^{\infty}\|^2\}.
\end{array}
$$
Thus, $0\in \partial g(w^{\infty}) - \mathcal{A} \nabla h(\mathcal{A}^* w^{\infty})$. The proof is completed.
\end{proof}
In fact, when the data matrix $S$ is positive definite, we can prove that $f$ is bounded below and the sequence $\{w^k\}$ has a limit point.
\begin{theorem}\label{thm:2.2}
When $S\in\mathbb{S}^n_{++}$, the level set of \eqref{cgl-mcp3} $\{w\,|\,f(w)\leq \alpha\}$ is bounded for every $\alpha \in\mathbb{R}$. It further implies that $f$ is bounded below and there exists a limit point of the sequence $\{w^k\}$ generated by Algorithm~\ref{alg-mm}.
\end{theorem}
\begin{proof}
See Appendix~\ref{app:A}.
\end{proof}

Theorem~\ref{thm-subcon} gives the subsequential convergence of Algorithm~\ref{alg-mm}. Sequential convergence and convergence rate can be established under  additional assumptions, including isolatedness of limit point and the KL (Kurdyka-\L{}ojasiewicz) property.  In fact, in the literature, they are many works on the analysis of the convergence rate of DCA and its variants; see, e.g., \cite{bolte2016majorization,attouch2009convergence,souza2016global,abbaszadehpeivasti2023rate,niu2022convergence}.
Next we establish sequential convergence and convergence rate results for Algorithm~\ref{alg-mm} with the assumption of the KL property. Our main result, Theorem~\ref{thm:2.6}, is mainly based on \cite[Proposition~4]{bolte2016majorization}. For completeness, we give its proof in the Appendix.
As a first step, we recall the definition of limiting subdifferential from \cite{rockafellar2009variational}.
\begin{definition}[Subdifferentials]
Let $f:\mathbb{R}^n \to (-\infty,+\infty]$ be a proper lower semicontinuous function and  $x\in {\rm dom} \, f := \{x\in \mathbb{R}^n  \mid f(x) < +\infty\}$. The Fr\'echet subdifferential of $f$ at $x$ is defined as
$$
\widehat{\partial}f(x) = \left\{v \in \mathbb{R}^n \,\,\Big|\,\, \liminf_{y\to x,y\neq x}\frac{f(y) - f(x) - \langle v,y-x\rangle}{\|x-y\|} \geq 0 \right\}.
$$
The limiting subdifferential of $f$ at $x$ is defined as
$$
\partial f(x)= \{ v \in \mathbb{R}^n \mid \exists \, x^k \to x,\, f(x^k) \to f(x),\,v^k \in \widehat{\partial}f(x^k),\, v^k \to v \mbox{ as } j \to \infty\}.
$$
\end{definition}
Next, we present the definition of the KL property \cite{bolte2007clarke,attouch2009convergence,bolte2007lojasiewicz}, which is a key tool in our convergence analysis. For $\alpha \in (0,+\infty]$, we denote by $\Phi_{\alpha}$ the class of functions $\phi:[0,\alpha) \to \mathbb{R}$ that satisfy the following conditions: (a) $\phi(0)=0$; (b) $\phi$ is positive, concave and continuous; (c) $\phi$ is continuously differentiable on $(0,\alpha)$, with $\phi' > 0$.
\begin{definition}[KL property]
Let $f:\mathbb{R}^n \to (-\infty,+\infty]$ be a proper lower semicontinuous function. The function $f$ is said to have the KL property at $\bar{x} \in {\rm dom} \,\partial f :=\{ x\in {\rm dom}\, f\mid \partial f(x) \neq \emptyset\}$, if there exist  $\alpha \in (0,+\infty]$, a neighborhood $V$ of $\bar{x}$ and a function $\phi \in \Phi_{\alpha}$ such that
$$
\phi'(f(x) - f(\bar{x})) {\rm dist}(0,\partial f(x)) \geq 1,\quad \forall\, x\in V \mbox{ and } f(\bar{x}) < f(x) < f(\bar{x}) + \alpha,
$$
where ${\rm dist}(x, C)=\min\{\|x-y\|\mid y\in C\}$ denotes the distance function to a closed set $C$. The function $f$ is said to be a KL function if it has the KL property at each point of ${\rm dom}\, \partial f$.
\end{definition}
When $\phi$ is of the form $\phi(s)=cs^{1-\theta}$ with $c>0$ and $\theta\in [0,1)$, the number $\theta$ is called a \L{}ojasiewicz exponent and we say $f$ has the KL property with  exponent $\theta$. Next we can give the sequential convergence and convergence rate results of Algorithm~\ref{alg-mm}.
\begin{theorem}\label{thm:2.6}
Suppose that the d.c. function $f$ in \eqref{cgl-mcp3} is bounded below. Assume that the sequence $\{\sigma_k\}$ is convergent to $\sigma_{\infty}>0$. Let $\{w^k\}$ be the sequence generated by Algorithm~\ref{alg-mm}.    Let $B^{\infty}$ be the set of limit points of $\{w^k\}$.
\begin{enumerate}
  \item[(I)] (Convergence result) The whole sequence $\{w^k\}$  converges to some critical point of \eqref{cgl-mcp3} if one of the following two conditions holds:
\begin{itemize}
  \item[(1)] $B^{\infty}$ contains an isolated point;

  \item[(2)]  $\{w^k\}$ is bounded, and for every $w^{\infty} \in B^{\infty}$, $f$ has the KL property at $w^{\infty}$.
\end{itemize}

\vspace{2mm}
\item[(II)] (Rate of convergence) If condition (2) holds, $w^k \to w^{\infty}$, and the function $f$ has the KL property at $w^{\infty}$ with exponent $\theta\in [0,1)$, then the following estimations hold:
    \begin{itemize}
      \item[(a)] if $\theta = 0$, then the sequence $\{w^k\}$  converges in a finite number of steps;
      \item[(b)] if $\theta \in (0,\frac{1}{2}]$, then there exit $c >0$ and $q \in [0,1)$ such that
          $$\|w^k-w^{\infty}\| \leq c q^k,\,\forall\,k\geq 1;$$
      \item[(c)] if $\theta \in (\frac{1}{2},1)$, then there exits  $c >0$ such that $$\|w^k-w^{\infty}\| \leq c k^{-(1-\theta)/(2\theta-1)},\,\forall\,k\geq 1.$$
    \end{itemize}
\end{enumerate}
\end{theorem}
\begin{proof}
See Appendix~\ref{app:A2}.
\end{proof}

\section{Semismooth Newton Method for the Subproblems}\label{sec-ssn}
In this section, we design a semismooth Newton (SSN) method for solving the subproblem \eqref{subprob}. Equivalently, we aim to solve the following problem for given $K$, $ \widetilde{\Theta}$, $\tilde{w}$,  and $\sigma$
\begin{equation}\label{subprob2}
  \begin{array}{cl}
  \min & -\log\det\, (\Theta + J) + \langle K,\Theta \rangle + \frac{\sigma}{2}\|\Theta - \widetilde{\Theta}\|^2 +  \frac{\sigma}{2}\|w - \tilde{w}\|^2 \\[6pt]
  {\rm s.t.} & \Theta = \mathcal{A}^*w,\\
  & w \in \mathbb{R}^{|\mathcal{E}|}_+.
  \end{array}
\end{equation}

\subsection{Properties of Proximal Mappings}
The properties of the proximal mappings associated with the log-determinant function and that of the indicator of the non-negative cone will be used subsequently in designing the SSN method. Therefore, we summarize the properties in this section.

Let $\mathcal{X}$ be a finite-dimensional real Hilbert space, $f:\, \mathcal{X}\rightarrow(-\infty,+\infty]$ be a closed proper convex function, and $\sigma>0$. The Moreau-Yosida regularization \cite{moreau1965proximite,yosida1980functional} of $f$ is defined by
\begin{equation}\label{def-MY}
\begin{array}{l}
\mathcal{M}_{f}^{\sigma}(x):=\min\limits_{z}\left\{f(z)+\frac{\sigma}{2}\|z-x\|^2\right\},\,\,\forall\, x\in \mathcal{X}.
\end{array}
\end{equation}
The unique optimal solution of \eqref{def-MY}, denoted by ${\rm Prox}_{f}^{\sigma}(x) $, is called the proximal point of $x$ associated with $f$, and ${\rm Prox}_{f}^{\sigma}:\, \mathcal{X}\rightarrow \mathcal{X}$ is called the associated proximal mapping.
Moreover, $\mathcal{M}_f^{\sigma}:\,\mathcal{X} \rightarrow \mathbb{R} $ is a continuously differentiable convex function  \cite{lemarechal1997practical,rockafellar2009variational}, and its gradient is given by
\begin{equation}\label{def-grad-MY}
\begin{array}{l}
\nabla \mathcal{M}_{f}^{\sigma} (x) = \sigma( x - {\rm Prox}_{f}^{\sigma}(x) ) ,\,\,\forall\,x\in\mathcal{X}.
\end{array}
\end{equation}

For notational simplicity, we denote
$\ell(\cdot) := -\log\det\, (\cdot)$.
As the log-determinant function $\ell$ is an important part in our problem, we  summarize the following results concerning the Moreau-Yosida regularization of $\ell$, see e.g.,
\cite[Lemma~2.1(b)]{wang2010solving}
\cite[Proposition~2.3]{yang2013proximal}.

\begin{proposition}\label{prop-logdet}
Let $\sigma > 0$. For any $X\in\mathbb{S}^n$ with its eigenvalue decomposition $X =  U\Lambda U^T$. Let $D$ be a diagonal matrix with $D_{ii} = (\sqrt{\Lambda_{ii}^2 + 4/\sigma} + \Lambda_{ii})/2$, $i=1,2,\dots,n$. Then we have that ${\rm Prox}_{\ell}^{\sigma}(X) = UDU^T$. Besides,
the proximal mapping ${\rm Prox}_{\ell}^{\sigma} :\,\mathbb{S}^n\to \mathbb{S}^n$ is continuously differentiable, and its directional derivative  at $X$ for any  $H\in\mathbb{S}^n$ is given by
$$
({\rm Prox}_{\ell}^{\sigma} )'(X)[H] = U[\Gamma \odot (U^T H U)]U^T,
$$
where $\odot$ denotes the Hadamard product, and
$\Gamma\in\mathbb{S}^n$ is defined by $$\Gamma_{ij} = \frac{1}{2} \left[ 1+ (\Lambda_{ii}+\Lambda_{jj})/(\sqrt{\Lambda_{ii}^2+4/\sigma}+\sqrt{\Lambda_{jj}^2+4/\sigma})\right] ,\,\,i,j=1,2,\dots,n.$$
\end{proposition}
Next, we handle the indicator function $\delta_+(\cdot)$ of the non-negative {cone} $\mathbb{R}_+^{|\mathcal{E}|}$ {defined by} $\delta_+(w)=  0$ if $w \in\mathbb{R}_+^{|\mathcal{E}|}$, and $\delta_+(w)=  +\infty$ otherwise.
For any given vector $c\in\mathbb{R}^{|\mathcal{E}|}$, the proximal mapping of $\delta_+(\cdot)$ is the projection onto the non-negative cone, which is given by
$$
\Pi_+(c) := {\rm Prox}_{\delta_+}^1(c) = \arg\min_w \left\{\delta_+(w) + \frac{1}{2}\| w - c\|^2\right\} = \max\{c,0\}.
$$
Here $\max\{\cdot,0\}$ is defined in a component-wise fashion such that $\max\{t,0\} = t$ if $t\geq 0$, and $\max\{t,0\} = 0$ if $t< 0$.
The Clarke generalized Jacobian \cite[Definition~2.6.1]{clarke1990optimization} of $\Pi_+:\,\mathbb{R}^{|\mathcal{E}|} \rightarrow \mathbb{R}^{|\mathcal{E}|}$ is given by
$$
\partial \Pi_+(c):= \Bigg\{  {\rm Diag}(d) \in \mathbb{S}^{|\mathcal{E}|}_+\,\Bigg|\,
\begin{array}{l}
d_i = 1, \mbox{ if } c_i>0;\\
0 \leq d_i \leq 1, \mbox{ if } c_i=0; \\
d_i = 0, \mbox{ if } c_i<0;
\end{array}
\forall\,i = 1,2,\dots,|\mathcal{E}| \Bigg\},\,\,\forall\, c\in\mathbb{R}^{|\mathcal{E}|}.
$$
We denote by ${\rm Diag}(d)$ the diagonal matrix whose diagonal elements are the components of $d$.

\subsection{Semismooth Newton Method for \eqref{subprob2}}
The dual of problem \eqref{subprob2} admits the following form
\begin{equation}\label{dual-sub}
\max_Y~~ \{\Phi (Y):= \underset{\Theta,w}{\min}~~ L(\Theta,w;Y)\},
\end{equation}
where $L$ is the Lagrangian function associated with \eqref{subprob2} and it is given by
\begin{equation}
\begin{array}{l}
L(\Theta,w;Y) \\[7pt]
= -\log\det\, (\Theta + J) + \langle K,\Theta \rangle + \frac{\sigma}{2}\|\Theta - \widetilde{\Theta}\|^2 +  \frac{\sigma}{2}\|w - \tilde{w}\|^2 + \langle\Theta - \mathcal{A}^*w,Y \rangle + \delta_+(w)\\[7pt]
=-\log\det(\Theta + J) + \frac{\sigma}{2}\|\Theta - \widetilde{\Theta} + \frac{1}{\sigma}(K+Y)\|^2  - \frac{1}{2\sigma}\|K+Y\|^2 + \langle \widetilde{\Theta},K+Y \rangle\\[7pt]
+\delta_+(w) + \frac{\sigma}{2}\|w - \tilde{w} - \frac{1}{\sigma}\mathcal{A}Y\|^2 - \frac{1}{2\sigma}\|\mathcal{A}Y\|^2 - \langle \tilde{w},\mathcal{A}Y \rangle,\,\,
(\Theta,w,Y)\in\mathbb{S}^n\times\mathbb{R}_+^{|\mathcal{E}|}\times\mathbb{S}^n.
\end{array}
\end{equation}
By the Moreau-Yosida regularization \eqref{def-MY} and \eqref{def-grad-MY}, we can obtain the following expression of $\Phi$ by minimizing $L$ w.r.t. $(\Theta,w)$
\begin{equation*}
\begin{array}{cl}
\Phi (Y) :=
&\mathcal{M}_{\ell}^{\sigma}(\widetilde{\Theta} + J - \frac{1}{\sigma}(K+Y)) + \sigma \mathcal{M}^1_{\delta_+}(\tilde{w}+\frac{1}{\sigma} \mathcal{A}Y)\\[7pt]
&\langle  Y,\widetilde{\Theta} -\mathcal{A}^* \tilde{w}\rangle - \frac{1}{2\sigma}\|K+Y\|^2- \frac{1}{2\sigma}\|\mathcal{A}Y\|^2 + \langle \widetilde{\Theta},K\rangle,
\end{array}
\end{equation*}
and it is continuously differentiable with the gradient
\begin{equation*}
\nabla \Phi (Y) = {\rm Prox}_{\ell}^{\sigma} (\widetilde{\Theta} + J - \frac{1}{\sigma}(K+Y)) - \mathcal{A}^*  \Pi_+( \tilde{w}+\frac{1}{\sigma} \mathcal{A}Y) -J.
\end{equation*}

To find the optimal solution of the unconstrained maximization  problem  \eqref{dual-sub}, we can equivalently
solve the following nonlinear nonsmooth equation
\begin{eqnarray}\label{nonsmooth-eqn}
  \nabla \Phi (Y)  \;=\; 0.
\end{eqnarray}
In this paper, we apply a globally convergent and locally superlinearly convergent SSN method \cite{kummer1988newton,qi1993nonsmooth,sun2002semismooth} to solve the
above nonsmooth equation.
To apply the SSN, we need the following generalized Jacobian of $\nabla \Phi$ at $Y$,
which is the multifunction $\widehat{\partial}^2\Phi (Y):\mathbb{S}^n\rightrightarrows\mathbb{S}^n$ defined as follows:
\begin{align}
\left\{
\begin{array}{l}
V\in\widehat{\partial}^2\Phi (Y) \hbox{ if and only if there exists $D\in\partial\Pi_+(\tilde{w}+\frac{1}{\sigma}\mathcal{A}Y)$ such that}\\[6pt]
V[H] =  - \frac{1}{\sigma}({\rm Prox}_{\ell}^{\sigma})' (\widetilde{\Theta} + J - \frac{1}{\sigma}(K+Y))[H]
- \frac{1}{\sigma}\mathcal{A}^*  D \mathcal{A} H,\,\,\forall\,H\in\mathbb{S}^n.
\end{array}
\right.
\end{align}
Each element $V\in\widehat{\partial}^2\Phi (Y)$ is negative definite.
The implementation of the SSN method is given in
Algorithm~\ref{alg-ssn}. For its convergence, see \cite[Theorem~3]{li2018efficiently}.
\begin{algorithm}[!ht]
\caption{Semismooth Newton method for solving \eqref{nonsmooth-eqn}}
\label{alg-ssn}
{\bf Input:} $\sigma> 0$, $(\widetilde{\Theta},\tilde{w})\in\mathbb{S}^n \times \mathbb{R}^{|\mathcal{E}|}$, $Y^0\in\mathbb{S}^{n}$, $\bar{\eta}\in (0,1)$, $\tau\in(0,1]$, $\mu\in (0,0.5)$, $\rho\in (0,1)$. $j=0$. Repeat until convergence

\begin{description}
  \item[Step 1.] Find
  $V_j\in\widehat{\partial}^2\Phi_k(Y^j).$
   Solve $V_j [D] = -\nabla \Phi(Y^j)$ inexactly to obtain $D^j$ such that $ \| V_j [D^j] + \nabla \Phi(Y^j) \| \leq \min(\bar{\eta}, \| \nabla \Phi(Y^j)\|^{1+\tau})$.

\vspace{2mm}
  \item[Step 2.]  Find $m_j$ as the smallest non-negative integer $m$ for which
  $ \Phi(Y^j + \rho^m D^j) \geq \Phi(Y^j) + \mu \rho^m \langle \nabla \Phi(Y^j),D^j\rangle.$ Set $\alpha_j = \rho^{m_j}$.
  
\vspace{2mm}
    \item[Step 3.] $Y^{j+1} = Y^j + \alpha_j D^j$. $j \leftarrow j+1$.
\end{description}
\end{algorithm}

Lastly, we can recover the optimal solution $(\overline{\Theta},\bar{w})$ to the primal problem \eqref{subprob2} via the optimal solution to the dual problem $\overline{Y}:=\arg\max\Phi(Y)$ by
\begin{equation*}
\overline{\Theta} = {\rm Prox}_{\ell}^{\sigma} (\widetilde{\Theta} + J - \frac{1}{\sigma}(K+\overline{Y})) -J,\quad
\bar{w} =  \Pi_+( \tilde{w}+\frac{1}{\sigma} \mathcal{A}\overline{Y}).
\end{equation*}

\begin{remark}
We claim that the stopping condition \eqref{stop-cond} is achievable under appropriate implementation  of Algorithm~\ref{alg-ssn}. For solving subproblem \eqref{subprob}, we implement  Algorithm~\ref{alg-ssn} with $\sigma = \sigma_k$, $\widetilde{\Theta} = \mathcal{A}^*w^k$,  $\tilde{w} = w^k$, and $K=G^k$, and then it will return an approximate solution $Y^{k+1}$ of \eqref{nonsmooth-eqn} with some error term $E^k:=-\nabla\Phi(Y^{k+1})$, namely,
\begin{equation}\label{eqn-rn1}
{\rm Prox}_{\ell}^{\sigma_k} (\mathcal{A}^*w^k + J - \frac{1}{\sigma_k}(G^k+Y^{k+1})) - \mathcal{A}^*  \Pi_+(  w^k+\frac{1}{\sigma_k} \mathcal{A}Y^{k+1}) -J + E^k=0.
\end{equation}
We let
\begin{equation}\label{eqn-rn2}
w^{k+1} :=  \Pi_+(w^k+\frac{1}{\sigma_k} \mathcal{A}Y^{k+1})
\end{equation}
be the solution to \eqref{subprob}. Next we justify that $w^{k+1}$ is an approximate solution to \eqref{subprob} satisfying the stopping condition \eqref{stop-cond} as long as  $\|E^k\|_2$ is smaller than {a} certain value, where $\|\cdot\|_2$ is the spectral norm of a matrix.

The equation \eqref{eqn-rn1} implies that  $ \mathcal{A}^*w^{k+1} + J-E^k\in\mathbb{S}^n_{++}$. We further require that $r:=\|(\mathcal{A}^*w^{k+1} + J-E^k)^{-1}E^k\|_2 < 1$ such that the matrix $\mathcal{A}^*w^{k+1}+ J$ is also positive definite \cite[Theorem~2.3.4]{golub2013matrix}.
We know that \eqref{eqn-rn1} is equivalent to that
\begin{equation}\label{eqn-rn3}
{\bf 0}\in\mathcal{A}^*(w^{k+1} - w^k) - E^k + \frac{1}{\sigma_k}(G^k+Y^{k+1}) + \frac{1}{\sigma_k}\partial \ell (\mathcal{A}^*w^{k+1} + J- E^k),
\end{equation}
and \eqref{eqn-rn2} is equivalent to that
\begin{equation}\label{eqn-rn4}
{\bf 0}\in w^{k+1} - w^k - \frac{1}{\sigma_k}\mathcal{A}Y^{k+1} + \partial \delta_+ (w^{k+1}).
\end{equation}
Since $\ell$ is differentiable on $\mathbb{S}^n_{++}$, we can obtain by  combining \eqref{eqn-rn3} and \eqref{eqn-rn4}  that
\begin{equation*}
\begin{array}{cl}
&  \mathcal{A}[ {\sigma_k} E^k   +   (\mathcal{A}^*w^{k+1} - E^k + J)^{-1}   -  (\mathcal{A}^*w^{k+1} + J)^{-1}]\\[7pt]
\in &  -  \mathcal{A} (\mathcal{A}^*w^{k+1} + J)^{-1} + \mathcal{A}G^k + \sigma_k(w^{k+1} - w^k) + \sigma_k \mathcal{A} \mathcal{A}^* (w^{k+1} - w^k) + \partial \delta_+ (w^{k+1}).
\end{array}
\end{equation*}
By noting the optimality condition of \eqref{subprob}, we can let the error vector in \eqref{stop-cond} be
\begin{equation*}
\delta^k:= - \mathcal{A}[ {\sigma_k}E^k  +  (\mathcal{A}^*w^{k+1} - E^k + J)^{-1}   -    (\mathcal{A}^*w^{k+1} + J)^{-1}].
\end{equation*}
We can observe that $\delta^k\to 0$ as $\|E^k\|_2\to 0$ and therefore the stopping condition is achievable. Specifically, it follows from  \cite[Theorem~2.3.4]{golub2013matrix} that
\begin{equation*}
\begin{array}{cl}
\|\delta^k\| &\leq  \|\mathcal{A}\|_2 \left[ {\sigma_k}  \|E^k\|_2 + \|(\mathcal{A}^*w^{k+1} - E^k + J)^{-1} -  (\mathcal{A}^*w^{k+1} + J)^{-1}\|_2\right]\\[7pt]
&\leq  \|\mathcal{A}\|_2 \|E^k\|_2\left[{\sigma_k } + \frac{\|(\mathcal{A}^*w^{k+1} - E^k + J)^{-1}\|_2^2}{1-r} \right],
\end{array}
\end{equation*}
where $\|\mathcal{A}\|_2:=\sup \{\|\mathcal{A}\Theta\|\,|\,\|\Theta\|_2\leq 1\}$. Therefore, the stopping condition \eqref{stop-cond} holds if the following two checkable inequalities hold:
\begin{equation*}
\begin{array}{ll}
r:=\|(\mathcal{A}^*w^{k+1} + J-E^k)^{-1}E^k\|_2 &< 1,\\[7pt]
 \|\mathcal{A}\|_2 \|E^k\|_2\left[{\sigma_k} + \frac{\|(\mathcal{A}^*w^{k+1} - E^k + J)^{-1}\|_2^2}{1-r} \right] &\leq \frac{\sigma_k}{4}\|w^{k+1} - w^k\| +
\frac{\sigma_k \| \mathcal{A}^* w^{k+1}- \mathcal{A}^*w^k\|^2 }{2\|w^{k+1} - w^k\|}.
\end{array}
\end{equation*}
\end{remark}

\section{Numerical Results}\label{sec-exp}
In this section, we conduct experiments to evaluate the performance of our graph learning model \eqref{cgl-mcp} from two perspectives. First of all, we compare the effectiveness of our model  \eqref{cgl-mcp} on synthetic data with the convex model \eqref{cgl-l1} and the generalized graph learning (GGL) model \cite[p.~828]{egilmez2017graph}. The GGL model refers to the problem $ \min \{-\log\det\, (\Theta ) + \langle S,\Theta \rangle + \lambda \|\Theta\|_{1,{\rm off}}\,|\,\Theta\in\mathcal{L}_g(A) \} $, where the set of generalized graph Laplacian $\mathcal{L}_g(A)$ is the set \eqref{set-lap2} without the row-sum constraint $\Theta{\bf 1}={\bf 0}$.
Secondly, we compare the efficiency of our inexact proximal DCA (Algorithm~\ref{alg-mm}) for solving \eqref{cgl-mcp} with the method in \cite[Algorithm~1]{ying2020nonconvex}, which solves a sequence of weighted $\ell_1$-norm penalized subproblems by a projected gradient descent algorithm\footnote{Codes are available at \url{https://github.com/mirca/sparseGraph}.}.
Our method is terminated if
\begin{equation*}
  \frac{\|w^k - w^{k-1}\|}{1 + \|w^{k-1}\|}<\varepsilon,\,\,\mbox{or}\,\,\,
  \frac{\|{\rm obj}^k - {\rm obj}^{k-1}\|}{1 + \|{\rm obj}^{k-1}\|}<\varepsilon.
\end{equation*}
Here, ${\rm obj}^k$ denotes the objective function value of \eqref{cgl-mcp3} at the $k$-th iteration.
The method \cite[Algorithm~1]{ying2020nonconvex} is terminated by their default condition, which  also uses certain relative successive changes.
In addition, we set the parameter $\gamma$ in \eqref{fct-MCP} to be $1.5$.

\subsection{Performance of Different Models}
In this section, we compare our model  \eqref{cgl-mcp}  with the model \eqref{cgl-l1} and the model (GGL). We refer to the solution of \eqref{cgl-l1} obtained by the alternating direction method of multipliers (see Appendix~\ref{app:B} for its implementation) as an L1 solution.
We call $\mathcal{A}^*\tilde{w}$ as an MCP solution, where $\tilde{w}$ is the approximate solution obtained  by our inexact proximal DCA presented in Algorithm~\ref{alg-mm}.

\subsubsection{Synthetic Graphs}\label{sec:syn-graphs}
We simulate graphs from three standard ensembles of random graphs:
1) Erd\H{o}s-R\'{e}nyi  $\mathcal{G}^{(n,p)}_{\rm ER}$, where two nodes are connected independently with probability $p$;
2) Grid graph, $\mathcal{G}^{(n)}_{\rm grid}$, consisting of $n$ nodes connected to their four nearest neighbours;
3) Random modular graph, $\mathcal{G}_M^{(n,p_1,p_2)}$, with $n$ nodes and four modules (subgraphs) where the probability of an edge connecting two nodes across modules and within modules are $p_1$ and $p_2$, respectively.
Given any graph structure from one of these three ensembles, the edge weights are uniformly sampled from the interval $[0.1,3]$. Then we obtain a valid Laplacian matrix as the true precision matrix.

\begin{figure}[htbp]
  \centering
  \includegraphics[width=0.9\textwidth]{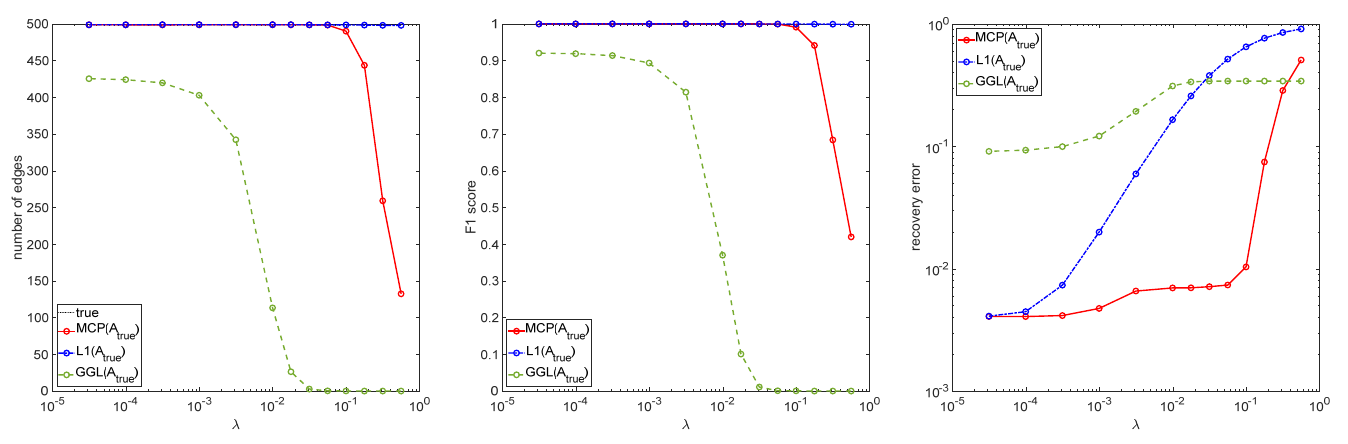}
  \caption{On  Erd\H{o}s-R\'{e}nyi graph, $\mathcal{G}^{(100,0.1)}_{\rm ER}$.  The true connectivity matrix  $A = A_{\rm true}$ is used.}\label{fig-er1}
\end{figure}
\begin{figure}[htbp]
  \centering
  \includegraphics[width=0.9\textwidth]{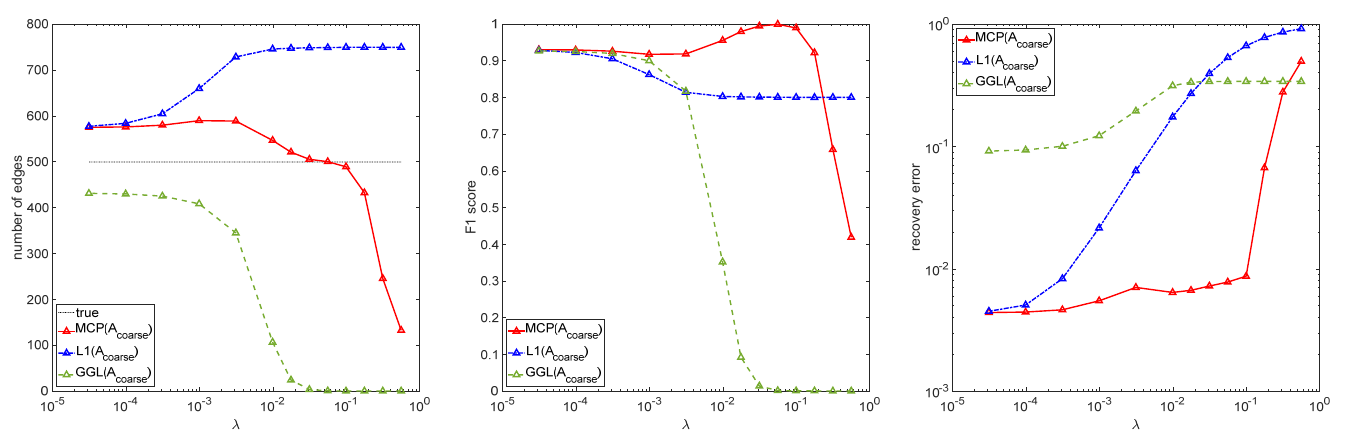}
  \caption{On  Erd\H{o}s-R\'{e}nyi graph, $\mathcal{G}^{(100,0.1)}_{\rm ER}$. {We use} a coarse estimation of the  true sparsity pattern and
  $A=A_{\rm coarse}$.}\label{fig-er2}
\end{figure}
\begin{figure}[htbp]
  \centering
  \includegraphics[width=0.9\textwidth]{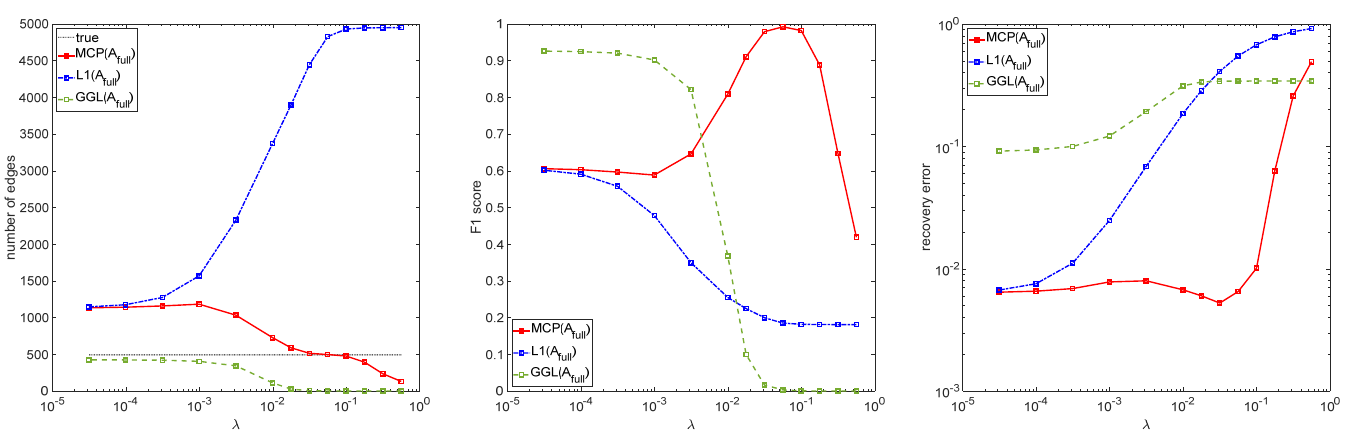}
  \caption{On  Erd\H{o}s-R\'{e}nyi graph, $\mathcal{G}^{(100,0.1)}_{\rm ER}$. {We use a full connectivity matrix} and  $A=A_{\rm full}$ is input.}\label{fig-er3}
\end{figure}
\begin{figure}[htbp]
  \centering
  \includegraphics[width=0.9\textwidth]{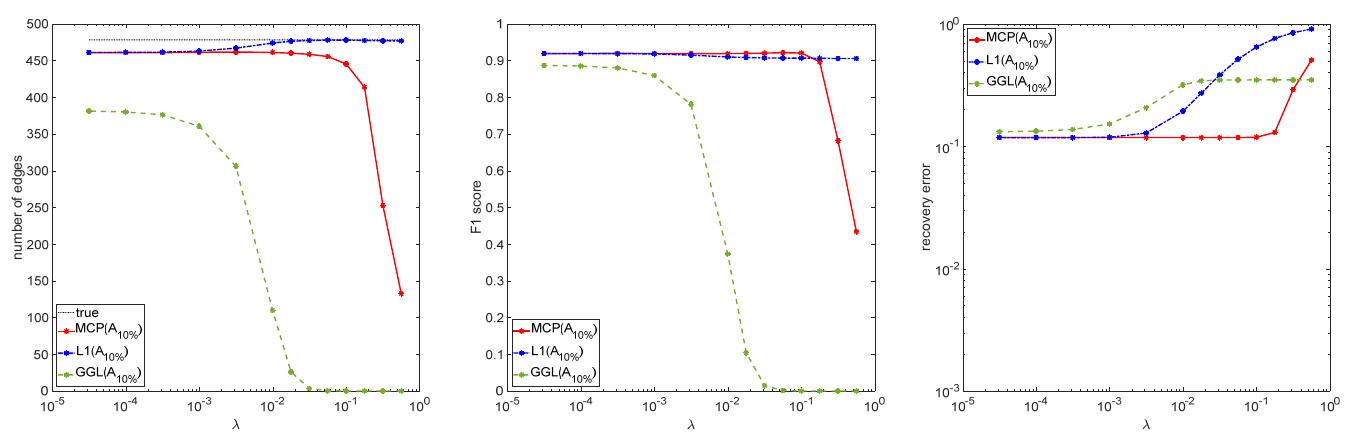}
  \caption{On  Erd\H{o}s-R\'{e}nyi graph, $\mathcal{G}^{(100,0.1)}_{\rm ER}$. {We use} a rough estimation of the true sparsity pattern and $A = A_{\rm 10\%}$ which is not exactly accurate.}\label{fig-er4}
\end{figure}

For the diversity of comparisons, we test the performance of different models  with various connectivity constraints: 1) $A = A_{\rm true}$, where $A_{\rm true}$ is the true connectivity matrix; 2) $A = A_{\rm coarse}$, where $A_{\rm coarse}$ is a coarse estimation of the truth, i.e., $\{(i,j)\,|\,(A_{\rm coarse})_{ij} =1 \}\supseteq \{(i,j)\,|\,(A_{\rm true})_{ij} =1 \}$. Specifically, in our experiments we set the cardinality of the former  as 1.5 times that of the latter by randomly changing some zero entries in $A_{\rm true}$ into ones; 3) $A = A_{\rm full}$, where $A_{\rm full}$ is the full connectivity matrix; 4) $A = A_{\rm d\%}$, where $A_{\rm d\%}$ is an inaccurate estimation of the truth and is obtained by randomly
replacing $d\%$ of the ones in  $A_{\rm true}$ by zero entries. Different connectivity constraints are reasonable since in some cases the true graph topology might not be available while one can obtain its coarse or inaccurate estimation based on some prior knowledge. Even without any prior knowledge, one can assume that the graph is fully connected and will estimate both the graph structure and edge weights. To measure the performance of different models, we adopt two metrics:
\begin{itemize}
\item[(1)]
recovery error $\frac{\|{\Theta} - L_{\rm true}\|}{\| L_{\rm true} \|}$, which is the relative error between the true precision matrix $L_{\rm true}$ and the estimated one $\Theta$;
\item[(2)]
F1 score $ \frac{2({\rm precision \cdot recall})}{{\rm precision + recall}} = \frac{2{\rm tp}}{2{\rm tp} + {\rm fp} + {\rm fn}} $, which is a standard metric to evaluate the performance {on} detecting edges. Here tp denotes true positive (the model correctly identifies an edge); fp denotes false positive (the model incorrectly identifies an edge); fn denotes false negative (the model fails to identify an edge).
\end{itemize}

We first compare L1, MCP, and GGL  solutions on  Erd\H{o}s-R\'{e}nyi graph, $\mathcal{G}^{(100,0.1)}_{\rm ER}$. We set $\varepsilon = 10^{-6}$ and the sample size $k=5000n$, and the results reported are the average over 10 simulations.
Fig.~\ref{fig-er1}---\ref{fig-er4} plot the number of edges, F1 score, and recovery error with respect to a sequence of $\lambda$ under different connectivity constraints.
As shown in Fig.~\ref{fig-er1}, with the true sparsity pattern, both MCP and L1 solutions can perfectly identify the edges; while GGL solutions only achieve the F1 score of 0.9 mainly due to its violation of the  constraint $\Theta{\bf 1}={\bf 0}$.
In terms of the recovery error which compares the edge weights of the true and estimated graphs, MCP solutions perform well while the other two tend to be biased for most values of $\lambda$. Fig.~\ref{fig-er1} shows that MCP solutions can be better than the other two models  for estimating the edge weights when the true sparsity pattern is given.
If the true connectivity matrix is unknown and only a coarse estimation of the connectivity matrix is available, we can see from Fig.~\ref{fig-er2} that only the MCP solution with $\lambda$ roughly in the interval $(10^{-2},10^{-1})$ can detect most of the edges and achieve the recovery error of $10^{-2}$. This has  demonstrated the ability of the MCP for promoting sparsity and avoiding bias. In addition, we can see from Fig.~\ref{fig-er3} that  without any estimation of the sparsity pattern, the results will deteriorate greatly compared to the results in Fig.~\ref{fig-er1} and \ref{fig-er2}. Even though the problem becomes much more difficult in this case, there still exists an MCP solution for which the F1 score is nearly 1, i.e., it can almost recover the true sparsity pattern. Meanwhile, the recovery error of the MCP solution is better than that of L1 or GGL solutions. We can also see from the left panel of Fig.~\ref{fig-er3} that the number of edges of L1 solutions change approximately from 1000 to 5000 when $\lambda$ changes from $10^{-4}$ to $10^{-1}$, and this number is always  larger than the  truth  value of around 500. The number of edges of the GGL solutions is slightly lower than the truth and tends to zero as $\lambda$ increases,
which can happen because GGL does not impose the row-sum constraint $\Theta{\bf 1} =0.$
Fig.~\ref{fig-er3} shows that without any prior knowledge of the true sparsity pattern, the MCP solutions are likely to be closer to the truth with a proper choice of  $\lambda$; L1 solution can hardly recover the ground truth; and GGL solutions can detect most of the  true edges when $\lambda$ is small but the solutions' values are greatly biased. Therefore, when prior knowledge about the connectivity matrix is not available, the MCP solutions are generally far more superior than the L1 or GGL solutions in terms of both structure inference and edge weights estimation.
As can be seen in Fig.~\ref{fig-er4}, the recovery error suffers from the incorrect prior information on the connectivity matrix, but the error is consistent with the percentage ($10\%$) of wrongly eliminated edges in the input connectivity matrix.
The results in Fig.~\ref{fig-er4} suggest that when the sample size is large enough ($k=5000n$), a fully connected prior connectivity matrix is preferable to an inaccurate estimation of the graph structure.

Additional numerical results for grid graph
$\mathcal{G}_{\rm grid}^{(100)}$ and random modular graph $\mathcal{G}_M^{(100,0.05,0.3)}$ are given in Appendix~\ref{app:C}.

\subsubsection{Real Data}\label{sec:4.1.2}
In this section, we test on a collection of real data sets: {\it animals}, {\it senate}, and {\it temperature}. Again, we run through a sequence of parameter values for $\lambda$ starting from a small scalar to a large enough number such that the resulting graph is almost fully connected or empty. We set $\varepsilon = 10^{-4}$.

\begin{figure}[htbp]
  \centering
  \includegraphics[width=1\textwidth]{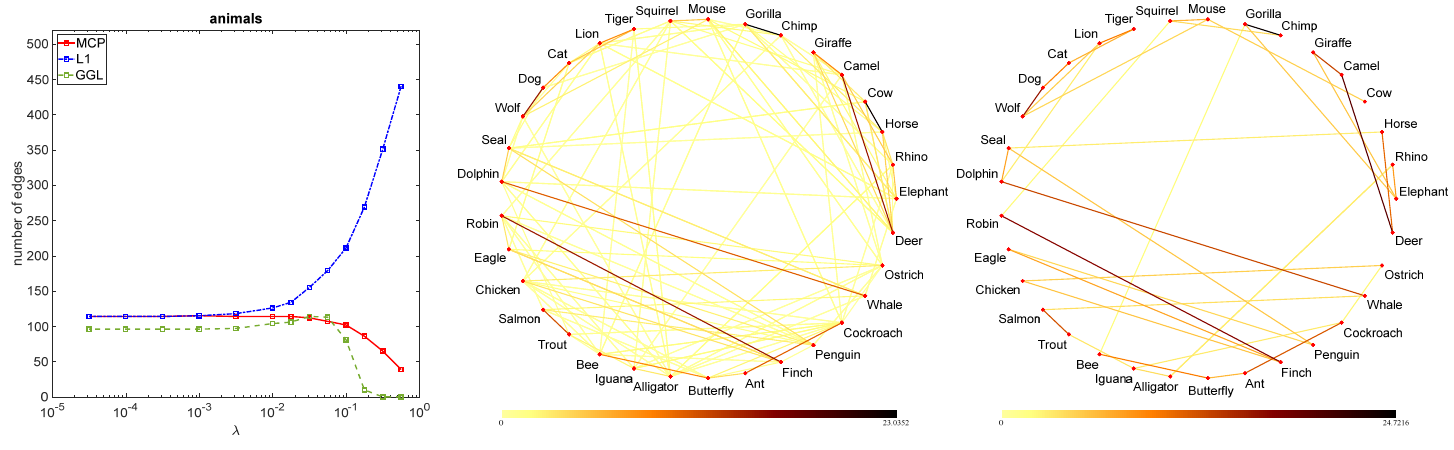}
  \caption{On {\it animals} data set. Left: the number of edges against the penalty parameter $\lambda$.
  Middle: dependency graph  of the MCP solution with  $\lambda=0$. Right: dependency graph  of the MCP solution with  $\lambda=10^{-0.25}$.}\label{graph-animals}
\end{figure}

\begin{figure}[htbp]
  \centering
  \includegraphics[width=1\textwidth]{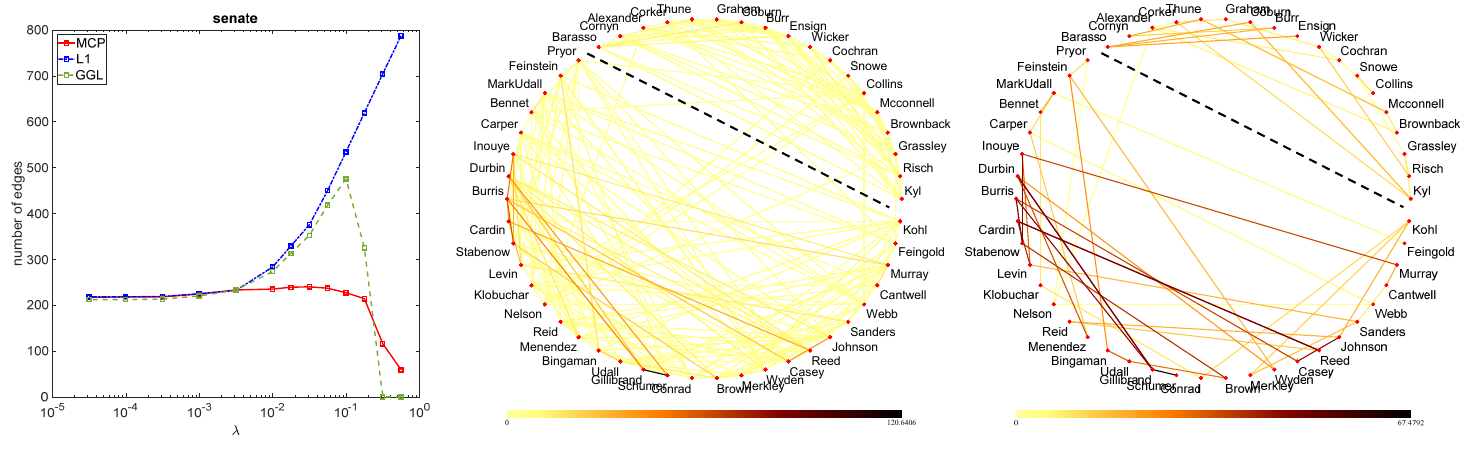}
  \caption{On {\it senate} data set. Left: the number of edges against the penalty parameter $\lambda$.
  Middle: dependency graph  of the MCP solution with  $\lambda=0$. Right: dependency graph  of the MCP solution with  $\lambda=10^{-0.25}$. The nodes at the top right (resp. bottom left) of the black dotted line represent Republicans (resp. Democrats).}\label{graph-senate}
\end{figure}

{\it  Animals:} The  {\it animals} data set \cite{kemp2008discovery} consists of binary values assigned to $k=102$ features for $n=33$ animals. Each feature denotes a true-false answer to a question, such as `has lungs?', `is warm-blooded?', `live in groups?'.
The left panel of Fig.~\ref{graph-animals} plots the number of edges against the penalty parameter $\lambda$ of the three models considered
in the previous subsection on the {\it animals} data set. The blue curve shows that increasing the penalty parameter $\lambda$ cannot promote sparsity in the L1 solutions, due to the presence of the constraint $\Theta{\bf 1}={\bf 0}$. In fact, when $\lambda$ is large, the majority of the edge weights $\Theta_{ij}$ of the L1 solution are small non-zero numbers satisfying the zero row-sum constraints. Therefore, the learned graph is almost fully connected  when $\lambda$ is large. On the other hand, we can see from the red curve that tuning the penalty parameter $\lambda$ will result in MCP solutions with various sparsity levels. It offers sparser solutions compared to L1 solutions, which are especially useful when the data contains a large number of nodes and a sparser graph is desired for better interpretability. Even though the ground truth is not available for real data, MCP solutions can provide solutions with a wider range of sparsity levels compared with L1 solutions and therefore MCP solutions will be preferable in real applications.
The middle and right panels of Fig.~\ref{graph-animals} illustrate the dependency networks  of the MCP solution without the regularization term ($\lambda=0$) and the sparsest dependency network among the MCP solutions ($\lambda=10^{-0.25}$), respectively. The right graph contains $39$ edges, which is more interpretable compared to the middle one with $114$ edges. We can clearly see from the right panel that similar animals, such as gorilla and chimp, dolphin and whale, are connected with edges of large weights, which coincides with one's expectations.

{\it Senate:} The {\it senate} data set \cite[Section~4.5]{lake2018emergence} contains 98 senators and their 696 voting records for  the 111th United States Congress, from January 2009 to January 2011. Similar to the {\it animals} data, the {\it senate} data consists of binary values, where 0's or 1's correspond to no or yes votes, respectively. There exist some missing entries when one senator is not present for certain votings.
To avoid missing entries, we select a submatrix ($50\times 293$) of the original data matrix ($98\times 696$) consisting of $50$ senators and $293$ voting records without missing entries. We then run through a sequence of $\lambda$ and plot the number of edges of estimated solutions in the left panel of Fig.~\ref{graph-senate}. Additionally, we illustrate the dependency networks of MCP solutions with $\lambda=0$ and $\lambda=10^{-0.25}$ (the sparsest one) in the middle and right panels of Fig.~\ref{graph-senate}, respectively. As can be seen, the middle graph, containing $218$ edges, is relatively dense and  not easy to interpret. In the right panel of Fig.~\ref{graph-senate}, the nodes at the bottom left (resp. top right) of the black dotted line represent Democrats (resp. Republicans).
The figure clearly shows the divide between Democrats and Republicans, and we can see that the two components are only connected by one edge between Democrat Nelson and Republican Corker.
In addition to the use of a small subset of the {\it senate} data to avoid missing data, we note that a generalized sample covariance matrix $S$ can be constructed according to the procedure in \cite[Equation~(2)]{kolar2012estimating} and \cite[Equation~(7)]{cai2016minimax}
to handle missing data in the area of inverse covariance estimation. Therefore, we can analyze the relationships among all of the $98$ senators. Table~\ref{tab:senate} compares the resulting graph without penalty ($\lambda=0$) and the sparsest graph offered by MCP solutions. We can see that without penalty, there are $2.45\%$ edges across the nodes representing Democrats and nodes representing Republicans. In contrast, the use of the MCP decreased this number to $1.56\%$. We believe that fewer edges across Democrats and Republicans might be more reasonable as the two parties are rarely correlated. Therefore, the \eqref{cgl-mcp} model can be a reasonable model to analyze the {\it senate} data.

{\it Temperature:} The {\it temperature} data set\footnote{NCEP Reanalysis data is provided by the NOAA/OAR/ESRL PSD, Boulder, Colorado, USA, from their Web site at \url{https://www.esrl.noaa.gov/psd/}.} we use contains the daily temperature measurements collected from $45$ states in the US over $16$ years (2000-2015) \cite{kalnay1996ncep}. Therefore, there are $k=5844$ samples for each of the $n=45$ states. Fig.~\ref{graph-temp} plots the number of edges of the graphs learned by three different models against the penalty parameter $\lambda$,  the dependency networks of the MCP solution without regularization ($\lambda=0$) and the sparsest one among the MCP solutions ($\lambda = 0.1$). As can be seen from the middle panel of Fig.~\ref{graph-temp}, the graph without regularization is quite sparse, and it is not much difference from the sparsest graph offered by MCP solutions shown in the right panel.
As plotted in the two networks, the states that are geographically close (especially contiguous) to each other are generally connected, since temperature values tend to be similar in nearby areas.
On the {\it temperature} data set, MCP and L1 solutions seem to be similar, as indicated by the left panel. One possible explanation could be that the {\it temperature} data admits certain sparsity intrinsically and it would result in a fairly sparse solution even without regularization.

\begin{table}[htbp]
  \centering
\begin{tabular}{lccc}
  \hline
   $\lambda$ & $({\rm nnz}_{\rm D},{{\rm nnz}_{\rm D}}/{{\rm nnz}})$ & $({\rm nnz}_{\rm R},{{\rm nnz}_{\rm R}}/{{\rm nnz}})$ & $({\rm nnz}_{\rm Cross},{{\rm nnz}_{\rm Cross}}/{{\rm nnz}})$ \\\hline
  0            & $(211,36.95\%)$ & $(346,60.60\%)$ & $(14,2.45\%)$ \\
  $10^{-0.25}$ & $(51,39.84\%)$ & $(75,58.59\%)$ & $(2,1.56\%)$ \\
  \hline
\end{tabular}
  \caption{nnz: the number of edges; ${\rm nnz}_{\rm D}$: the number of edges connecting two nodes representing Democrats; ${\rm nnz}_{\rm R}$: the number of edges connecting two nodes representing Republicans; ${\rm nnz}_{\rm Cross}$: the number of edges connecting one node representing Democrats and one node representing Republicans.}\label{tab:senate}
\end{table}

\begin{figure}[htbp]
  \centering
  \includegraphics[width=1\textwidth]{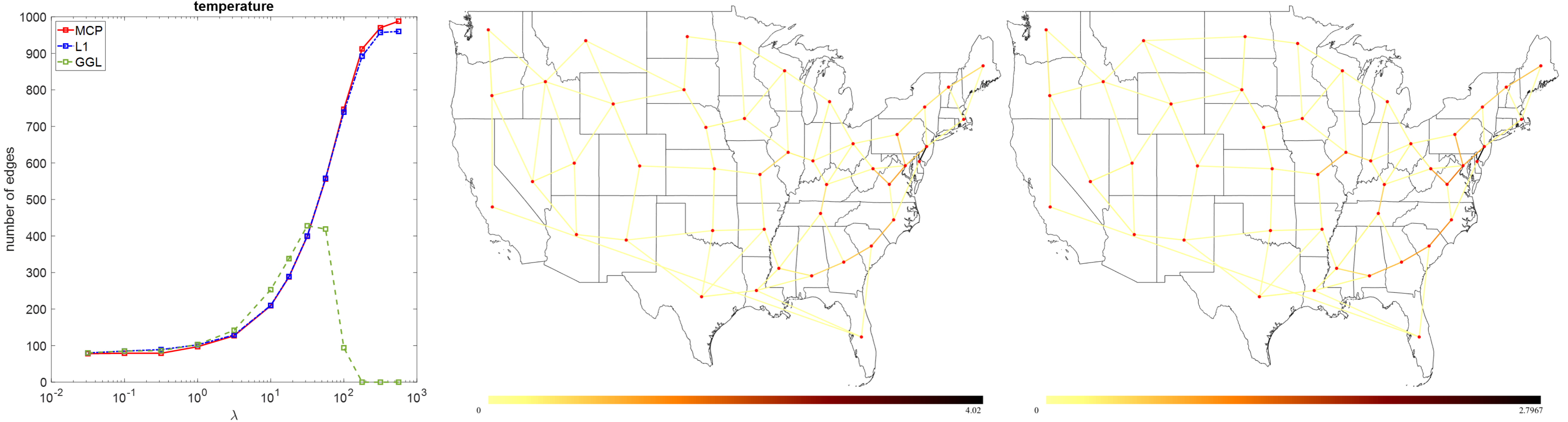}
  \caption{On {\it temperature} data set. Left: the number of edges against the penalty parameter $\lambda$.
  Middle: dependency graph  of the MCP solution with  $\lambda=0$. Right: dependency graph  of the MCP solution with  $\lambda=10^{-1}$.}\label{graph-temp}
\end{figure}

\subsection{Computational Efficiencies of Different Methods}
In this section, we compare our inexact proximal DCA (Algorithm~\ref{alg-mm}) with the method in \cite[Algorithm~1]{ying2020nonconvex}, which is implemented in a R package referred to as `NGL', for solving a special case of \eqref{cgl-mcp} for which $A$ is restricted to be the full connectivity matrix.

\subsubsection{Synthetic Graphs}
We generate graphs from the ensemble of random modular graphs described in section \ref{sec:syn-graphs} to evaluate the efficiency of the two methods. We set $p_1=0.005$, $p_2 = 0.25$, $\lambda = 0.005$ and $\varepsilon = 10^{-6}$. The graphs are of different dimensions $n\in\{160,240,320,400\}$ and the sample size $k$ is set to be $5000n$. Note that our algorithm can use the information of the true graph topology $A$; while the NGL will not incorporate $A$ in their solver. For a fair comparison, we merely input a full connectivity matrix into our algorithm.

\begin{table}[!h]\centering
\setlength{\tabcolsep}{1.2mm}{\small
\begin{tabular}{lccccccccrr}
\toprule
Nodes & \multicolumn{2}{c}{Edges} & \multicolumn{2}{c}{Time} & \multicolumn{2}{c}{F1 score} & \multicolumn{2}{c}{Recovery error} & \multicolumn{2}{c}{Objective value}\\
\cmidrule(l){2-3} \cmidrule(l){4-5} \cmidrule(l){6-7} \cmidrule(l){8-9} \cmidrule(l){10-11}
$n$ & DCA & NGL & DCA & NGL & DCA & NGL & DCA & NGL & DCA & NGL \\
\midrule
160 &829  & 1078  &20.3 & 15.0 & 0.99 & 0.85 & 7.3e$-$03 & 5.1e$-$03 & $-$2.5607e+02 & $-$2.5566e+02  \\
240 &1684 & 23832 &12.6 & 0.8  & 0.94 & 0.14 & 1.7e$-$02 & 1.0e+01   & $-$6.3314e+01  & 2.2850e+03  \\
320 &2810 & --    &15.9 & --   & 0.91 & --   & 2.2e$-$02 & --        & $-$1.0257e+02  & \multicolumn{1}{c}{--}\\
400 &4149 & --    &46.6 & --   & 0.89 & --   & 2.7e$-$02 & --        & $-$1.4570e+02  & \multicolumn{1}{c}{--}           \\
\bottomrule
\end{tabular}
}
\caption{Performances of DCA and NGL on modular graph $\mathcal{G}_M^{(n,0.005,0.25)}$. $\lambda=0.005$. $\varepsilon = 10^{-6}$. `--' means the method stops due to internal errors.}\label{table-mod-eff}
\end{table}

\begin{table}[!h]\centering
\setlength{\tabcolsep}{1.2mm}{\small
\begin{tabular}{clccccrc}
\toprule
Problem & $\lambda$ & \multicolumn{2}{c}{Edges} & \multicolumn{2}{c}{Time} & \multicolumn{2}{c}{Objective value}\\
\cmidrule(l){3-4} \cmidrule(l){5-6} \cmidrule(l){7-8}
&  & DCA & NGL & DCA & NGL & DCA & NGL  \\
\midrule
                    &  $10^{-0.5}$ &65 & 53 &1.1 & 0.3 &   $-$3.7797e+01 & $-$3.8267e+01  \\
                    &  $10^{-1.0}$ &105 & 68 &1.3 & 0.5 &  $-$4.6493e+01 & $-$4.6232e+01  \\
{\it animals}       &  $10^{-1.5}$ &115 & 85 &1.0 & 1.1 &  $-$4.7896e+01 & $-$4.7759e+01  \\
$(n=33)$            &  $10^{-2.0}$ &118 & 96 &0.8 & 1.0 &  $-$4.8053e+01 & $-$4.8070e+01  \\
                    &  $10^{-2.5}$ &119 & 102 &0.4 & 1.3 & $-$4.8052e+01 & $-$4.8115e+01  \\[0.2cm]
                    &  $10^{-0.5}$ &119 &-- &5.1  &--  & $-$9.3342e+01 &--    \\
                    &  $10^{-1.0}$ &243 &-- &10.3 &--  & $-$1.1016e+02 &--    \\
{\it senate}        &  $10^{-1.5}$ &250 &-- &8.7  &--  & $-$1.1345e+02 &--    \\
$(n=50)$            &  $10^{-2.0}$ &245 &-- &5.8  &--  & $-$1.1383e+02 &--    \\
                    &  $10^{-2.5}$ &232 &-- &6.8  &--  & $-$1.1409e+02 &--    \\[0.2cm]
                         &  $10^{0.5}$  &127 &-- &5.7 &--  & 1.2958e+02 &--    \\
                         &  $10^{0}$    &97  &-- &2.0 &--  & 1.0489e+02 &--    \\
{\it temperature}        &  $10^{-0.5}$ &79  &-- &3.5 &--  & 8.8216e+01 &--    \\
$(n=45)$                 &  $10^{-1.0}$ &78  &-- &3.8 &--  & 8.3571e+01 &--    \\
                         &  $10^{-1.5}$ &78  &-- &4.2 &--  & 8.2770e+01 &--    \\[0.2cm]
                    &  $10^{0}$    &7956 &-- &62.6  &--  & 8.5260e+02 &--    \\
                    &  $10^{-0.5}$ &1631 &-- &370.9 &--  & 3.6811e+02 &--    \\
{\it lymph}         &  $10^{-1.0}$ &2014 &-- &296.8 &--  & 2.0798e+02 &--    \\
$(n=587)$           &  $10^{-1.5}$ &3155 &-- &226.2 &--  & 1.7707e+02 &--    \\
                    &  $10^{-2}$   &3949 &-- &136.4 &--  & 1.7248e+02 &--    \\[0.2cm]
                    &  $10^{0}$    &8805 &-- &111.6 &--  & 9.4494e+02  &--    \\
                    &  $10^{-0.5}$ &1514 &-- &897.1 &--  & 5.8620e+01  &--    \\
{\it estrogen}      &  $10^{-1.0}$ &2145 &-- &521.3 &--  & $-$1.0362e+02 &--    \\
$(n=692)$           &  $10^{-1.5}$ &2939 &-- &511.9 &--  & $-$1.3415e+02 &--    \\
                    &  $10^{-2}$   &3461 &-- &369.5 &--  & $-$1.3833e+02 &--    \\
\bottomrule
\end{tabular}
}
\caption{Performances of DCA and NGL on real data. $\varepsilon = 10^{-4}$. `--' means the method stops due to internal errors.}\label{table-animals-eff}
\end{table}

Table~\ref{table-mod-eff} compares the computational time, F1 score, recovery error, and objective value of problem \eqref{cgl-mcp2}  of the two methods DCA and NGL for solving instances with different numbers of nodes $n$.
Table~\ref{table-mod-eff} shows that our DCA can successfully solve all instances with satisfactory F1 score and recovery error. In contrast, the NGL only succeeded in solving the problem when $n=160$; and it terminated prematurely when $n=240$ as shown by the fairly low F1 score, unreasonably large recovery error and objective value. For relatively large dimensions with $n=320$ and $n=400$, the NGL does not return reasonable solutions as it terminates prematurely due to internal errors caused by the singularity of certain matrices. It seems that the NGL might not be reliable for solving  the model \eqref{cgl-mcp}  on random modular graphs. From Table~\ref{table-mod-eff} we can conclude that our inexact proximal DCA is fairly efficient for solving the model  on random modular graphs.

\subsubsection{Real Data}
We compare the two methods on the real data set {\it animals}, {\it senate}, and {\it temperature}. We encountered numerical issue due to  matrix singularity when the NGL is applied for solving the {\it senate} and {\it temperature} data sets. We run through a sequence of parameters $\lambda$ used in Section~\ref{sec:4.1.2} which result in sparse graphs.
Table~\ref{table-animals-eff} compares the number of edges of the resulting graphs, the computational time, and the objective value of problem \eqref{cgl-mcp2} of the two methods. It can be seen that both DCA and NGL can solve the instances on the {\it animals} data within several seconds, and the objective values  are comparable.
For {\it senate} and {\it temperature} data, only DCA can return reasonable solutions.

In addition, we test on genetic real data sets {\it lymph} ($n=587$)  and {\it estrogen} ($n=692$) from \cite[Section~4.2]{li2010inexact}, and the model \eqref{cgl-mcp} can extract dependency relationships among the genes.
The results are also presented in Table~\ref{table-animals-eff}. Due to the relatively large dimensions, the computational time on genetic real data sets increased correspondingly compared to the {\it animals}, {\it senate}, and {\it temperature} data sets.

The numerical results in Tables~\ref{table-mod-eff} and \ref{table-animals-eff} show that our inexact proximal DCA is a fairly efficient and robust method for solving the \eqref{cgl-mcp} model.

\section{Conclusion}\label{sec-con}
In this paper, we have designed an inexact proximal DCA for solving the MCP penalized graphical model with Laplacian structural constraints \eqref{cgl-mcp}.
We also prove that any limit point of the sequence generated by
the inexact proximal DCA is a critical point of \eqref{cgl-mcp}. Each subproblem of the proximal DCA is solved by an efficient semismooth Newton method. Numerical experiments have demonstrated the effectiveness of the model \eqref{cgl-mcp} and the efficiency of the inexact proximal DCA, together with the semismooth Newton method, of solving the model. More generally, both the model and algorithm can be applied  directly to other non-convex penalties, such as the smoothly clipped absolute deviation (SCAD) function.

\section*{Funding}
The first author was supported by  the National Natural Science Foundation of China under grant number 12201617,
the second author was supported by the Academic Research Fund of the Ministry of Education of Singapore under grant number MOE2019-T3-1-010, and
the last author was supported by Hong Kong Research Grant Council under grant number 15304019.

\bibliographystyle{tfs}
\bibliography{GLbib}

\newpage
\appendix

\section{Proof of Theorem~\ref{thm:2.2}}\label{app:A}

\begin{lemma}\label{lem:A1}
Consider the following problem
\begin{equation}\label{cgl-l1a}
\min~~\{-\log\det\,(\mathcal{A}^*w + J) + \langle S,\mathcal{A}^*w\rangle + \delta(w\,|\,\mathbb{R}^{|\mathcal{E}|}_+) \}.
\end{equation}
{\rm (i)}
The level set \eqref{lev-set} of problem \eqref{cgl-l1a}
\begin{equation}\label{lev-set}
\{w\in\mathbb{R}^{|\mathcal{E}|}\,|\,-\log\det\,(\mathcal{A}^*w + J) + \langle S,\mathcal{A}^*w\rangle + \delta(w\,|\,\mathbb{R}^{|\mathcal{E}|}_+)\leq \alpha\}
\end{equation}
is closed for every $\alpha\in\mathbb{R}$. Namely, the essential objective function of problem \eqref{cgl-l1a}
\begin{equation*}
g(w) := -\log\det\, (\mathcal{A}^*w + J) + \langle S,\mathcal{A}^*w \rangle +  \delta(w\,|\,\mathbb{R}^{|\mathcal{E}|}_+),\,w\in\mathbb{R}^{|\mathcal{E}|}
\end{equation*}
is lower semi-continuous on $\mathbb{R}^{|\mathcal{E}|}$.
\\[5pt]
{\rm (ii)}
Suppose the condition $\{ w\in\mathbb{R}^{|\mathcal{E}|}_+\,|\,\langle S,\mathcal{A}^*w  \rangle \leq 0\} = \{{\bf 0}\}$ holds (it holds if the given matrix $S$ is positive definite). Then the level set \eqref{lev-set} of problem \eqref{cgl-l1a} is bounded for every $\alpha \in \mathbb{R}$ and the solution set of problem \eqref{cgl-l1a} is a non-empty bounded set.
\end{lemma}
\begin{proof}
{\rm (i)} It suffices to prove the condition that
\begin{equation}\label{proof:1}
\alpha \geq -\log\det\,(\mathcal{A}^*w + J) + \langle S,\mathcal{A}^*w\rangle + \delta(w\,|\,\mathbb{R}^{|\mathcal{E}|}_+)
\end{equation}
whenever $\alpha = \lim \alpha_k$ and $w = \lim w_k$ for sequences $\{\alpha_k\}$ and $\{w_k\}$ such that $\alpha_k \geq -\log\det\,(\mathcal{A}^*w_k + J) + \langle S,\mathcal{A}^*w_k\rangle + \delta(w_k\,|\,\mathbb{R}^{|\mathcal{E}|}_+)$ for every $k$.
If $\alpha = +\infty$, \eqref{proof:1} holds automatically. Next we focus on the case where $\alpha$ is finite. We claim that $\nu := \lambda_{\rm min}(\mathcal{A}^*w + J) >0$ if $\alpha$ is finite.

Proof of the claim: If $\lambda_{\rm min}(\mathcal{A}^*w + J) = 0$, then for every $i$, there exists $k_i$ such that
$\lambda_{\rm min}(\mathcal{A}^*w_{k_i} + J) < \frac{1}{i}$ and
$\lambda_{\rm max}(\mathcal{A}^*w_{k_i} + J) < \lambda_{\rm max}(\mathcal{A}^*w + J) + 1$. Then, $-\log\det\,(\mathcal{A}^*w_{k_i} + J) > -(n-1) \log\,(\lambda_{\rm max}(\mathcal{A}^*w + J) + 1) + \log\,i$. By letting $i\to+\infty$, we obtain that $\alpha = +\infty$, which is contradictory to the finiteness of $\alpha$. Therefore, the claim is proved.

By the continuity of $-\log\det(\cdot)$ on $\{\Theta\,|\,\Theta \succeq \nu I\}$ and the closedness of the set $\mathbb{R}^{|\mathcal{E}|}_+$, we can obtain \eqref{proof:1} by letting $k \to +\infty$. The lower semi-continuity of $g$ follows from \citep[Theorem~7.1]{rockafellar1996convex}.

{\rm (ii)}
By \citep[Theorem~8.5]{rockafellar1996convex},  the recession function $g0^+$ of $g$ is given as follows: for $w\in\mathbb{R}^{|\mathcal{E}|}_+$,
\begin{equation*}
\begin{array}{cl}
(g0^+)(w)&=\displaystyle\lim\limits_{t\to+\infty} \frac{g({\bf 1} + t w) - g({\bf 1})}{t}\\
&=\displaystyle \lim\limits_{t\to+\infty} \frac{-\log\det\,(\mathcal{A}^*{\bf 1} + J + t\mathcal{A}^*w) + \log\det\,(\mathcal{A}^*{\bf 1} + J)}{t} + \langle S,\mathcal{A}^*w\rangle.
\end{array}
\end{equation*}
Since $\mathcal{A}^*w\in\mathbb{S}^n_+$, it is easy to obtain that
\begin{equation*}
\lim\limits_{t\to+\infty} \frac{-\log\det\,(\mathcal{A}^*{\bf 1} + J + t\mathcal{A}^*w) + \log\det\,(\mathcal{A}^*{\bf 1} + J)}{t}=0.
\end{equation*}
Therefore, $(g0^+)(w)=\langle S,\mathcal{A}^*w\rangle$ if  $w\in\mathbb{R}^{|\mathcal{E}|}_+$; $(g0^+)(w)=+\infty$ otherwise. The recession cone of $g$ is given by
$\{w\,|\,(g0^+)(w)\leq 0\} = \{w\in\mathbb{R}^{|\mathcal{E}|}_+\,|\,\langle S,\mathcal{A}^*w\rangle\leq 0\} $. Then (ii) follows from \citep[Theorem~8.4, Theorem~8.7, \& Theorem~27.1]{rockafellar1996convex} and the proof is completed.
\end{proof}

\begin{proof}[Proof of Theorem~\ref{thm:2.2}]
The boundedness of the level set follows from Lemma~\ref{lem:A1} (ii) and the following inclusion
\begin{equation*}
\begin{array}{cl}
\{w\,|\,f(w)\leq \alpha\} &= \{ w\in\mathbb{R}^{|\mathcal{E}|}_+\,|\, -\log\det\,(\mathcal{A}^*w + J) + \langle S,\mathcal{A}^*w\rangle + P(\mathcal{A}^*w) \leq \alpha \}\\
&\subseteq \{ w\in\mathbb{R}^{|\mathcal{E}|}_+\,|\, -\log\det\,(\mathcal{A}^*w + J) + \langle S,\mathcal{A}^*w\rangle  \leq \alpha \}.
\end{array}
\end{equation*}
The rest results follow from \citep[Theorem~1.9]{rockafellar2009variational} and that Algorithm~\ref{alg-mm} is a descent algorithm. The proof is completed.
\end{proof}

\section{Proof of Theorem~\ref{thm:2.6}}\label{app:A2}

\begin{proof}[Proof of Theorem~\ref{thm:2.6}]
We first prove (I). Suppose that there exists $k_0$ such that $w^{k_0+1} = w^{k_0}$,  implying $\delta^{k_0}=0$, it follows from the optimality condition of \eqref{subprob} that $w^{k_0}$ is a critical point of \eqref{cgl-mcp3}. Now we suppose that $\|w^{k+1} - w^k\|>0$ for all $k\geq 0$.

We show in Theorem~\ref{thm:2.2} that $\lim\limits_{k\to \infty} \|w^{k+1} - w^k\|=0$. If condition (1) holds, then the convergence of $\{w^k\}$ follows immediately from \cite[Proposition~8.3.10]{facchinei2003finite}.

Now we assume condition (2) holds. It follows from \eqref{major-f} and \eqref{subprob-eq} that
$
0\in \partial g(w^{k+1}) -\mathcal{A} \nabla h(\mathcal{A}^*w^k) + \sigma_k(w^{k+1} - w^k) + \sigma_k\mathcal{A}\mathcal{A}^*(w^{k+1} - w^k) + \delta^k.
$
Namely,
\begin{align*}
\xi^{k+1} &:=\mathcal{A} \nabla h(\mathcal{A}^*w^k) - \mathcal{A} \nabla h(\mathcal{A}^*w^{k+1}) - \sigma_k(w^{k+1} - w^k) - \sigma_k\mathcal{A}\mathcal{A}^*(w^{k+1} - w^k) - \delta^k \\
& \in \partial g(w^{k+1}) - \mathcal{A} \nabla h(\mathcal{A}^*w^{k+1}) \subseteq \partial f(w^{k+1}).
\end{align*}
Moreover, we can see from the definition \eqref{grad-h} that $\nabla h(\cdot)$ is globally Lipschitz continuous. Therefore, there  exists $K>0$ such that $\|\xi^{k+1}\|\leq K \|w^{k+1} - w^k\|$ for all $k$. That is,
\begin{equation}\label{pf1}
  {\rm dist}(0,\partial f(w^{k+1})) \leq K \|w^{k+1} - w^k\| \mbox{  for all } k.
\end{equation}
Since $\{w^k\}$ is bounded under condition (2), there exits a limit point $w^{\infty} \in B^{\infty}$ of $\{w^k\}$. By Lemma~\ref{lemma2.1}, we have that $\lim\limits_{k\to \infty}  f(w^k)=f(w^{\infty})$. Without loss of generality, we assume that $f(w^{\infty})=0$. Since $f$ has the KL property at $w^{\infty}$, there exit $\delta> 0$, $\alpha \in (0,+\infty]$, and $\phi \in \Phi_{\alpha}$ such that $\phi'(f(w)) {\rm dist}(0,\partial f(w)) \geq 1$ for all $w$ such that $\|w - w^{\infty}\|\leq \delta $ and $0< f(w) < \alpha$.

By Lemma~\ref{lemma2.1} and the fact that $\sigma_k \geq \sigma_{\infty}$, we obtain that
\begin{equation}\label{pf2}
f(w^{k+1}) \leq f(w^k) - \frac{\sigma_{\infty}}{4} \|w^{k+1} - w^k\|^2
\leq f(w^k) - \frac{\sigma_{\infty}}{4}\frac{\|w^{k+1} - w^k\|^2}{\|w^{k} - w^{k-1}\|}\|w^{k} - w^{k-1}\|.
\end{equation}
By \eqref{pf1}, we further have that
$$
f(w^{k+1}) \leq f(w^k) - \frac{\sigma_{\infty}}{4K}\frac{\|w^{k+1} - w^k\|^2}{\|w^{k} - w^{k-1}\|} {\rm dist}(0,\partial f(w^{k})).
$$

Let $r\geq s \geq 1$ be some integers and assume that the points $w^s,\dots,w^r$ belong to $B(w^{\infty},\delta):=\{w \mid \|w - w^{\infty}\|\leq \delta\}$ with $f(w^s) < \alpha$. From the monotonicity and concavity of $\phi$, we obtain that
$$
\phi(f(w^{k+1})) \leq \phi(f(w^{k})) - \phi'(f(w^{k}))\frac{\sigma_{\infty}}{4K}\frac{\|w^{k+1} - w^k\|^2}{\|w^{k} - w^{k-1}\|} {\rm dist}(0,\partial f(w^{k})),
$$
thus by the KL property, for $k\in\{s,\dots,r\}$,
\begin{equation}\label{pf4}
\phi(f(w^{k+1})) \leq \phi(f(w^{k})) - \frac{\sigma_{\infty}}{4K}\frac{\|w^{k+1} - w^k\|^2}{\|w^{k} - w^{k-1}\|}.
\end{equation}
By the inequality $2(a-b) \geq \frac{a^2-b^2}{a},\,\forall\,a>0,\,b\in\mathbb{R}$ and \eqref{pf4}, we have for  $k\in\{s,\dots,r\}$,
\begin{align*}
\|w^{k} - w^{k-1}\| & = \frac{\|w^{k+1} - w^k\|^2}{\|w^{k} - w^{k-1}\|} + \frac{\|w^{k} - w^{k-1}\|^2 - \|w^{k+1} - w^k\|^2}{\|w^{k} - w^{k-1}\|}\\
&\leq \frac{\|w^{k+1} - w^k\|^2}{\|w^{k} - w^{k-1}\|}  + 2( \|w^{k} - w^{k-1}\| - \|w^{k+1} - w^k\|)\\
&\leq \frac{4K}{\sigma_{\infty}} (\phi(f(w^{k})) - \phi(f(w^{k+1}))) + 2( \|w^{k} - w^{k-1}\| - \|w^{k+1} - w^k\|)
\end{align*}
Hence, by summation
\begin{equation}\label{pf9}
\sum_{k=s}^{r}\|w^{k} - w^{k-1}\| \leq
\frac{4K}{\sigma_{\infty}} (\phi(f(w^{s})) - \phi(f(w^{r+1}))) + 2( \|w^{s} - w^{s-1}\| - \|w^{r+1} - w^r\|).
\end{equation}

Next we show the sequence remains in the neighborhood and converges. Since $\phi$ is continuous and $f(w^k) \downarrow 0$, we can find a sufficiently large $N$ such that
\begin{align}
  \|w^N - w^{\infty}\| & \leq \frac{\delta}{4} \label{pf5}\\
   \frac{4K}{\sigma_{\infty}} \phi(f(w^{N}))& \leq \frac{\delta}{4} \label{pf6}\\
  \sqrt{\frac{4}{\sigma_{\infty}} f(w^{N})} & < \min\left(\frac{\delta}{4},\sqrt{\frac{4\alpha}{\sigma_{\infty}}}\right)\label{pf7}
\end{align}
It follows from \eqref{pf2} that
\begin{equation}\label{pf8}
\|w^{k+1} - w^k\| \leq
\sqrt{ \frac{4}{\sigma_{\infty}} (f(w^k) - f(w^{k+1})) }
\leq \sqrt{ \frac{4}{\sigma_{\infty}} f(w^k) }
<\frac{\delta}{4},\quad \forall\,k\geq N.
\end{equation}
Let us prove that $w^r\in B(w^{\infty},\delta)$ for $r\geq N$. We proceed by induction on $r$. By \eqref{pf5}, $w^N\in B(w^{\infty},\delta)$. By \eqref{pf7}, $f(W^N)< \alpha$. Suppose that $r\geq N+1$, and $w^N,\dots,w^{r-1} \in B(w^{\infty},\delta)$, then
\begin{align*}
  \|w^r - w^{\infty}\| & \leq \|w^r - w^{r-1}\| + \|w^{r-1} - w^{N}\| + \|w^{N} - w^{\infty}\| \\
  & < \frac{\delta}{4} + \sum_{k=N+1}^{r-1}\|w^{k} - w^{k-1}\|  + \frac{\delta}{4} \\
  & \leq \frac{\delta}{2} + \frac{4K}{\sigma_{\infty}} (\phi(f(w^{N+1})) - \phi(f(w^{r}))) + 2( \|w^{N+1} - w^{N}\| - \|w^{r} - w^{r-1}\|)\\
  & \leq \frac{\delta}{2} + \frac{4K}{\sigma_{\infty}}  \phi(f(w^{N+1}))  + 2 \|w^{N+1} - w^{N}\| \leq \delta,
\end{align*}
where the second inequality follows from \eqref{pf5} and \eqref{pf8}; the third inequality follows from \eqref{pf9}; and the last inequality follows from \eqref{pf6} and \eqref{pf8}. Hence $w^N,\dots,w^{r} \in B(w^{\infty},\delta)$ and the induction proof is complete. Therefore, $w^r\in B(w^{\infty},\delta)$ for $r\geq N$. Using \eqref{pf9} again, we obtain that the series $\sum \|w^{k} - w^{k-1}\|$ converges, hence $w^k$ also converges by Cauchy's criterion.

The second part (II) is proved as in \cite[Theorem~2]{attouch2009convergence}. Here $\phi$ can be chosen of the form $\phi(s)=cs^{1-\theta}$ with $c>0$ and $\theta\in [0,1)$. Then \eqref{pf9}, the KL property, and \eqref{pf1} yield a similar result as in \cite[(11)]{attouch2009convergence}, which therefore leads to the same estimates.
\end{proof}

\section{ADMM for Solving \eqref{cgl-l1}}\label{app:B}
In this part, we briefly describe the alternating direction method of multipliers (ADMM) for solving \eqref{cgl-l1} and refer the readers to  \citep{chen2017efficient,fazel2013hankel} for its convergence properties. First we reformulate the model \eqref{cgl-l1} as follows:
\begin{equation}\label{cgl-l1-2}
  \begin{array}{cl}
  \min\limits_{\Theta,w,x} & -\log\det\, (\Theta + J) + \langle K,\Theta \rangle \\[6pt]
  {\rm s.t.} & \Theta = \mathcal{A}^*x,\\
  & w-x=0,\\
  & \Theta \in \mathbb{S}^n,\,\,w \in \mathbb{R}^{|\mathcal{E}|}_+,\,\,x \in \mathbb{R}^{|\mathcal{E}|},
  \end{array}
\end{equation}

where $K:=S+\lambda I$.
It is easy to derive the following dual problem of \eqref{cgl-l1-2}:
\begin{equation}\label{cgl-l1-dual}
  \begin{array}{cl}
  \max & \log\det\, (Y + K ) - \langle  J, Y+K \rangle + n \\[5pt]
  {\rm s.t.} &  \mathcal{A}Y + \zeta = 0,\\[3pt]
  & \zeta \in \mathbb{R}^{|\mathcal{E}|}_+, \;\; Y \in \mathbb{S}^n.
  \end{array}
\end{equation}
The Karush-Kuhn-Tucker (KKT) optimality conditions associated with \eqref{cgl-l1-2} and \eqref{cgl-l1-dual} are given as follows:
\begin{equation}\label{kkt}
\begin{array}{l}
\Theta - \mathcal{A}^* x = 0, \;\; w-x = 0, \;\; w \in \mathbb{R}^{|\mathcal{E}|}_+, \\[3pt]
\mathcal{A}Y + \zeta = 0,\;\; \zeta \in \mathbb{R}^{|\mathcal{E}|}_+,\\[3pt]
(\Theta + J)(Y+K) = I,\quad \Theta+J\in\mathbb{S}^n_{++}, \quad \langle w,\zeta\rangle = 0.
\end{array}
\end{equation}
The iteration scheme of our ADMM for solving \eqref{cgl-l1-2} can be described as follows: given $\tau\in(0,(1+\sqrt{5})/2)$, and an initial point $(x^0,\Theta^0,w^0,Y^0,\zeta^0)$, the $(k+1)$-th iteration is given by
\begin{equation}\label{method-admm}
\left\{
\begin{array}{l}
x^{k+1} = (I+\mathcal{A}\mathcal{A}^*)^{-1}[\mathcal{A}(\Theta^k + {\sigma_k}^{-1}Y^k) + w^k + {\sigma_k}^{-1}\zeta^k],\\[6pt]
\Theta^{k+1} = {\rm Prox}^{\sigma_k}_{\ell}(J + \mathcal{A}^*x^{k+1} -{\sigma_k}^{-1} Y^k - {\sigma_k}^{-1} K) - J,\\[6pt]
w^{k+1} = \Pi_+(x^{k+1} - {\sigma_k}^{-1}\zeta^k),\\[6pt]
Y^{k+1} = Y^k + \tau\sigma_k (\Theta^{k+1} - \mathcal{A}^*x^{k+1}),\,\zeta^{k+1} = \zeta^k + \tau\sigma_k(w^{k+1} - x^{k+1}).
\end{array}\right.
\end{equation}
We measure the optimality of an estimated primal-dual solution obtained from ADMM by the {relative KKT residual} $\max\{\eta_p,\eta_d,\eta_g\}$, where
\begin{equation*}
\begin{array}{ll}
{\rm pobj} = -\log\det\,( \mathcal{A}^*w + J) + \langle K,\mathcal{A}^*w \rangle,
& {\rm dobj} = \log\det\,(Y+K) - \langle J,Y+K\rangle + n,\\[3pt]
\eta_p = \max\left\{ \frac{  \max\{ \|\Theta - \mathcal{A}^*x\|, \|w-x\| \} }{1+\|x\|}, \frac{\|\Pi_+(-w)\|}{1+\|w\|} \right\},
& \eta_d =  \frac{\max \{ \|\mathcal{A}Y + \zeta\|,  \|\Pi_+(-\zeta)\|\}}{1+\|\zeta\|},\\[3pt]
\eta_g = |{\rm pobj} - {\rm dobj}|/(1 + |{\rm pobj}| + |{\rm dobj}|).&
\end{array}
\end{equation*}
The ADMM is terminated if $\max\{\eta_p,\eta_d,\eta_g\} < \varepsilon$, for a given tolerance $\varepsilon > 0$.

Next, we discuss efficient techniques to solve the $|\mathcal{E}| \times |\mathcal{E}|$ linear system in the first step of the above iteration: $(I+\mathcal{A}\mathcal{A}^*)x=b$, for any given $b\in \mathbb{R}^{|\mathcal{E}|}$. Obviously, the linear system can be solved inexactly by an iterative method such as the conjugate gradient method. However, when the linear system  is of moderate dimension, say $|\mathcal{E}| < 5000$, it is generally more efficient to solve it by a direct method via a pre-computed Cholesky decomposition. Since the direct method requires the explicit matrix form of the linear map $\mathcal{A}\mathcal{A}^*$, we derive its matrix representation in the following proposition.

\begin{proposition}\label{prop-AAt}
Let $\mathcal{G}=(V,\mathcal{E})$ be a given graph with $|V|=n$, and  $B\in\mathbb{R}^{n\times |\mathcal{E}|}$ be the node-arc incidence matrix of $\mathcal{G}$. We define a linear map $\mathcal{A}^*:\,\mathbb{R}^{|\mathcal{E}|} \to \mathbb{S}^n$ as
$
\mathcal{A}^*w = B{\rm Diag}(w)B^T,\,w\in\mathbb{R}^{|\mathcal{E}|}.
$
Then the matrix representation of $\mathcal{A}\mathcal{A}^*:\,\mathbb{R}^{|\mathcal{E}|} \to\mathbb{R}^{|\mathcal{E}|} $ is $2I + |B|^T|B|$.
\end{proposition}
\begin{proof}
By the property of incidence matrices, we know that the diagonal entries of $B^TB$ are $2$ and the off-diagonal entries of $B^TB$ are $0$ or $\pm 1$. Thus we can split $B^TB$ into two parts: $B^TB = 2I + C$. Note that $ C:= B^TB-2I$ has all its diagonal entries equal to $0$ and the off-diagonal entries are either $0$ or $\pm 1$.
For any $w\in\mathbb{R}^{|\mathcal{E}|}$, we have that
$$
\mathcal{A}\mathcal{A}^*w = {\rm diag}( B^T B{\rm Diag}(w)B^T B)
= 4w +{\rm diag}(- 2C {\rm Diag}(w) - 2{\rm Diag}(w)C + C{\rm Diag}(w)C). $$
It follows from simple computations that
${\rm diag}(C{\rm Diag}(w)C) = (C\odot C)w$, where $\odot$ denotes elementwise product.
Together with the fact that ${\rm diag}(C) = 0$, we have
$ \mathcal{A}\mathcal{A}^*w = 4w + (C\odot C) w$. By noting the properties of $B$ and $C$, we can deduce that  $ C\odot C = |C| = |B^TB| - 2I = |B|^T|B| - 2I$, where $|\cdot|$ means taking elementwise absolute value. Therefore, $ \mathcal{A}\mathcal{A}^*w = (2I + |B|^T|B|) w,\,\forall\,w$.
The proof is completed.
\end{proof}
By Proposition~\ref{prop-AAt}, we know that the linear system $(I+\mathcal{A}\mathcal{A}^*)x=b$ becomes
\begin{equation}\label{eqn-2}
(3I + |B|^T|B|)x=b.
\end{equation}
In the case where the number of edges $|\mathcal{E}|$ is moderate (say $|\mathcal{E}| < 5000$), we can solve the equation \eqref{eqn-2} exactly
by computing the Cholesky decomposition of the sparse  matrix $3I + |B|^T|B|$. The sparse Cholesky decomposition
will merely be performed once at the beginning of the ADMM. With the
pre-computed Cholesky decomposition, the solution of \eqref{eqn-2}
can be computed via solving of two  triangular systems of linear equations with $O(|\mathcal{E}|^2)$ operations. In it is the case that $|\mathcal{E}|$  is large but the number of nodes $n$ is moderate (say $n < 5000$), we can use the Sherman-Morrison-Woodbury formula \citep[(2.1.3)]{golub2013matrix} to get:
\begin{equation*}
(3I + |B|^T|B|)^{-1} = \frac{1}{3}\big(I - |B|^T (3I + |B||B|^T)^{-1}|B|\big).
\end{equation*}
Therefore, to solve \eqref{eqn-2}, one only needs to solve an $n\times n$ linear system of the form $(3I + |B||B|^T)\tilde{x} = \tilde{b}$.
Moreover, it is easy to see that the coefficient matrix has the same sparsity pattern as the Laplacian matrix of the graph defined by $B$, which is likely to be sparse for our problem. For the case when both $|\mathcal{E}|$ and $n$ are too  large to perform efficient Cholesky decompositions, we have to resort to an iterative solver such as the conjugate gradient method  to solve \eqref{eqn-2}. Each iteration of the conjugate gradient method requires the multiplication of $3I + |B|^T|B|$ by a vector in $\mathbb{R}^{|\mathcal{E}|}$. By taking advantage of the sparsity of $B$, each matrix-vector multiplication requires $O(|\mathcal{E}|)$ operations.

\section{Additional Numerical Results}\label{app:C}
This section presents additional numerical results to complement those in Section~\ref{sec:syn-graphs}.
We test on  grid graph, $\mathcal{G}_{\rm grid}^{(100)}$ and random modular graph, $\mathcal{G}_M^{(100,0.05,0.3)}$. We set the sample size $k=5000n$ and the results are the average over 10 simulations.
Fig.~\ref{fig-grid1}---\ref{fig-grid4} plot the number of edges, F1 score, and recovery error with respect to a sequence of $\lambda$ with different connectivity constraints on $\mathcal{G}_{\rm grid}^{(100)}$.
Fig.~\ref{fig-mod1}---\ref{fig-mod4} plot the number of edges, F1 score, and recovery error with respect to {a sequence of} $\lambda$ with different connectivity constraints on $\mathcal{G}_M^{(100,0.05,0.3)}$.

\begin{figure}[!h]
  \centering
  \includegraphics[width=0.9\textwidth]{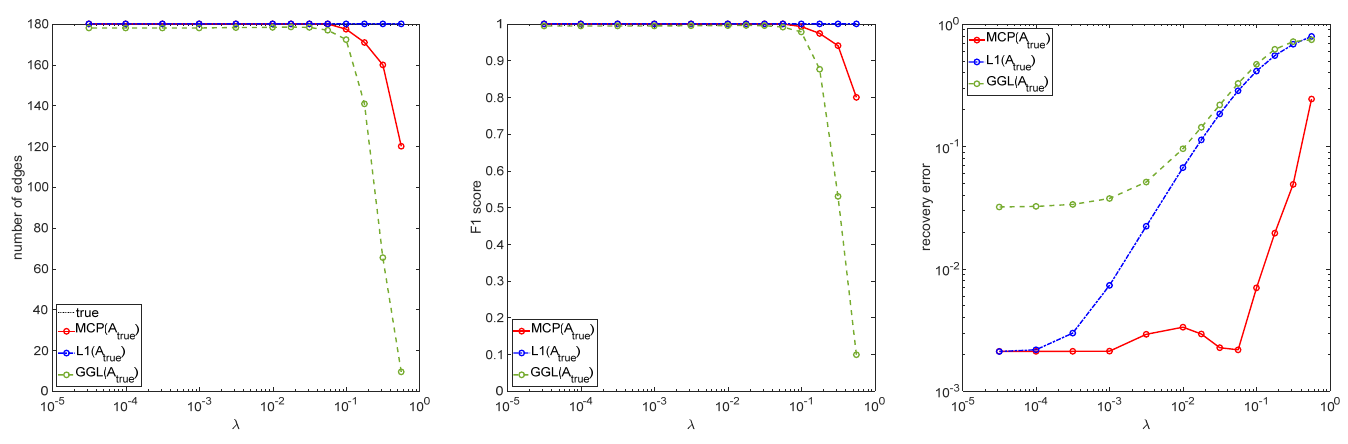}
  \caption{Results for grid graph, $\mathcal{G}_{\rm grid}^{(100)}$.  Scenario: the true connectivity matrix  $A = A_{\rm true}$ is used.}\label{fig-grid1}
\end{figure}
\begin{figure}[!h]
  \centering
  \includegraphics[width=0.9\textwidth]{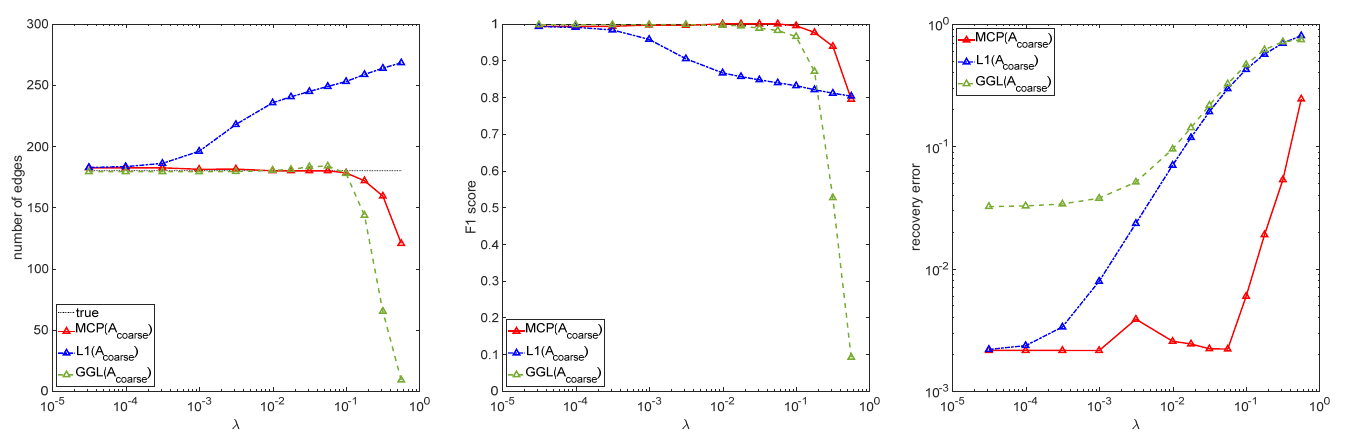}
  \caption{Results for  grid graph, $\mathcal{G}_{\rm grid}^{(100)}$. Scenario: we have a coarse estimation of the  true sparsity pattern and $A=A_{\rm coarse}$.}\label{fig-grid2}
\end{figure}
\begin{figure}[!h]
  \centering
  \includegraphics[width=0.9\textwidth]{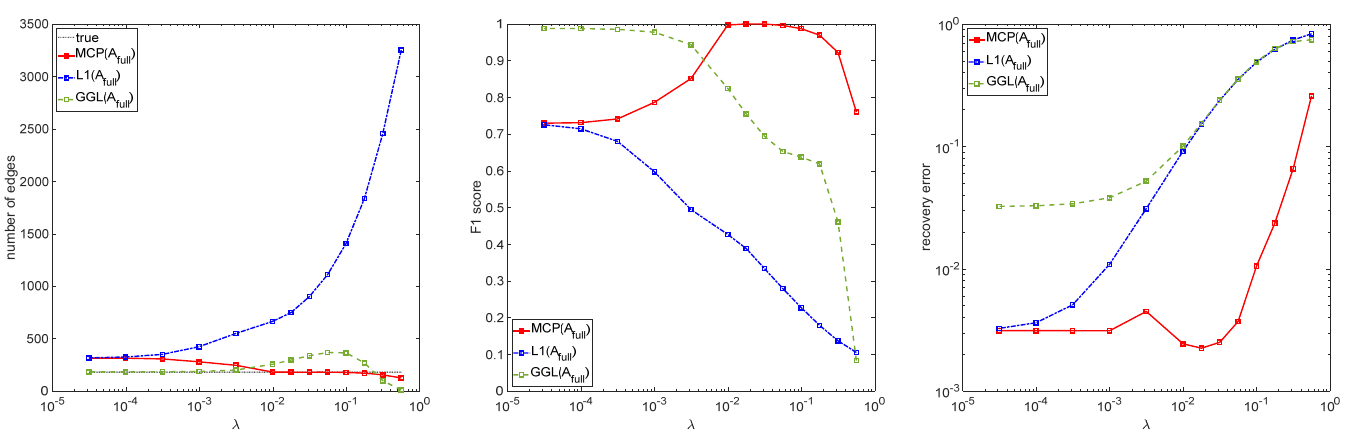}
  \caption{Results for  grid graph, $\mathcal{G}_{\rm grid}^{(100)}$. Scenario: the true sparsity pattern is unknown and  $A=A_{\rm full}$ is the input.}\label{fig-grid3}
\end{figure}
\begin{figure}[!h]
  \centering
  \includegraphics[width=0.9\textwidth]{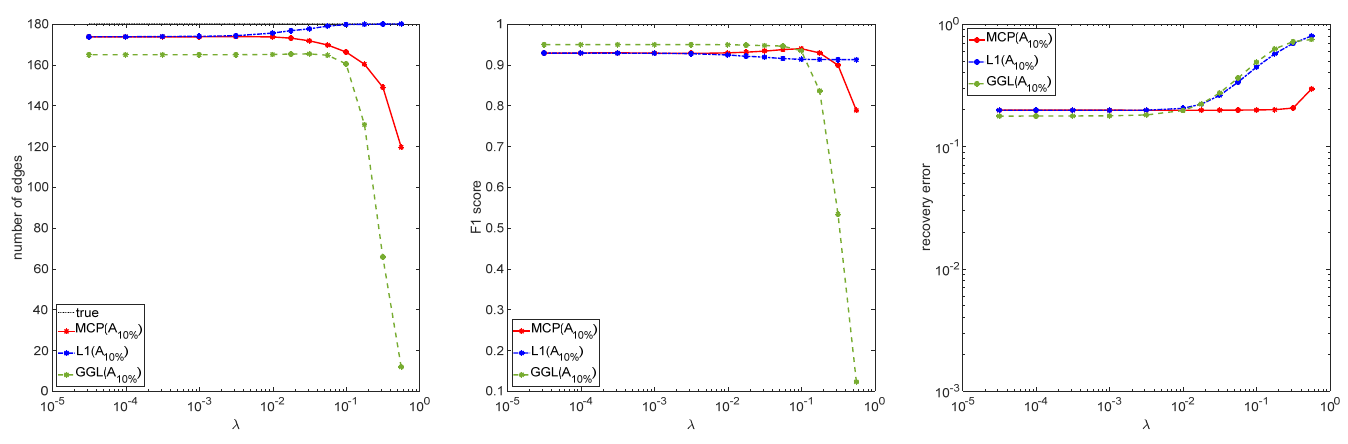}
  \caption{Results  grid graph, $\mathcal{G}_{\rm grid}^{(100)}$. Scenario: we have a rough estimation of the true sparsity pattern and $A = A_{\rm 10\%}$ is not exactly accurate.}\label{fig-grid4}
\end{figure}
\begin{figure}[!h]
  \centering
  \includegraphics[width=0.9\textwidth]{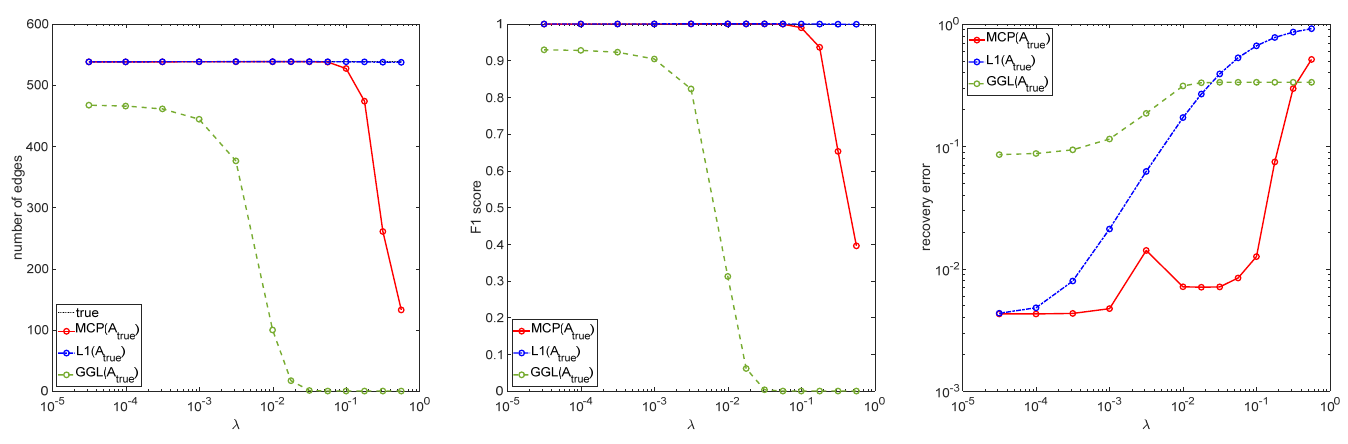}
  \caption{Results for  random modular graph, $\mathcal{G}_M^{(100,0.05,0.3)}$.
  Scenario: the true connectivity matrix  $A = A_{\rm true}$ is used.}\label{fig-mod1}
\end{figure}
\begin{figure}[!h]
  \centering
  \includegraphics[width=0.9\textwidth]{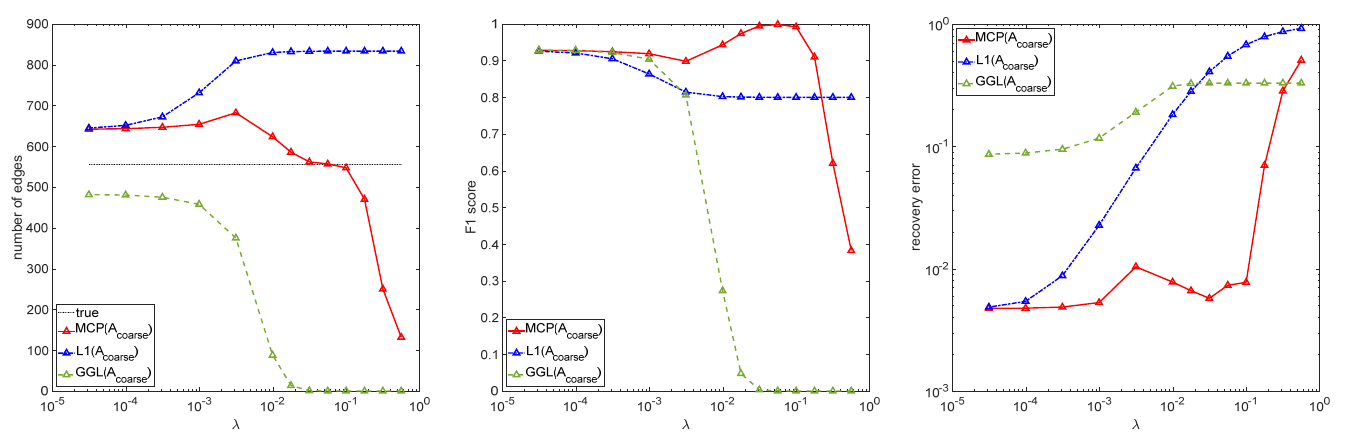}
  \caption{Results for  random modular graph, $\mathcal{G}_M^{(100,0.05,0.3)}$.
  Scenario: we have a coarse estimation of the  true sparsity pattern and $A=A_{\rm coarse}$.}\label{fig-mod2}
\end{figure}
\begin{figure}[!h]
  \centering
  \includegraphics[width=0.9\textwidth]{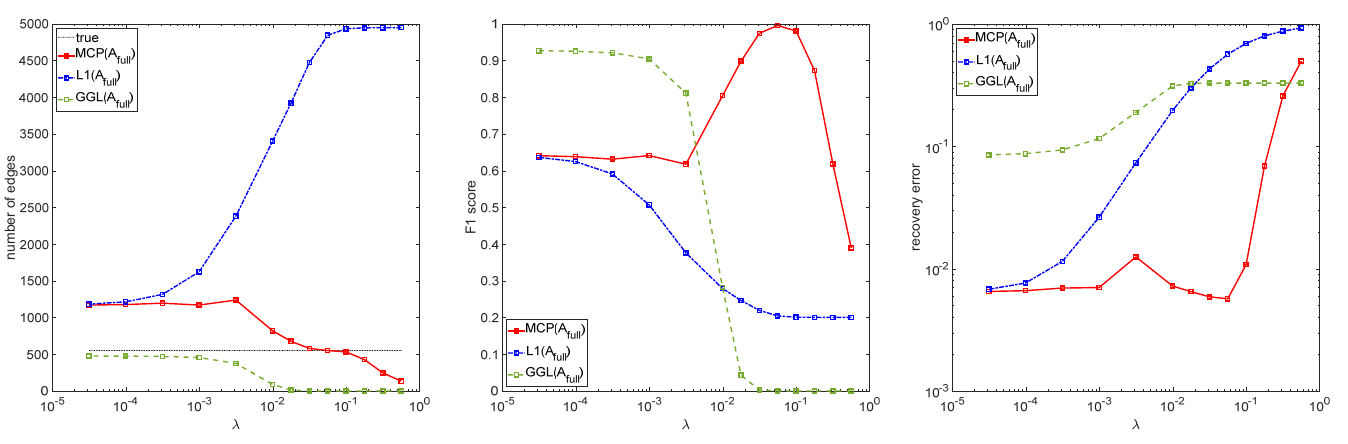}
  \caption{Results for random modular graph, $\mathcal{G}_M^{(100,0.05,0.3)}$.
  Scenario: the true sparsity pattern is unknown and  $A=A_{\rm full}$ is the input.}\label{fig-mod3}
\end{figure}
\begin{figure}[!h]
  \centering
  \includegraphics[width=0.9\textwidth]{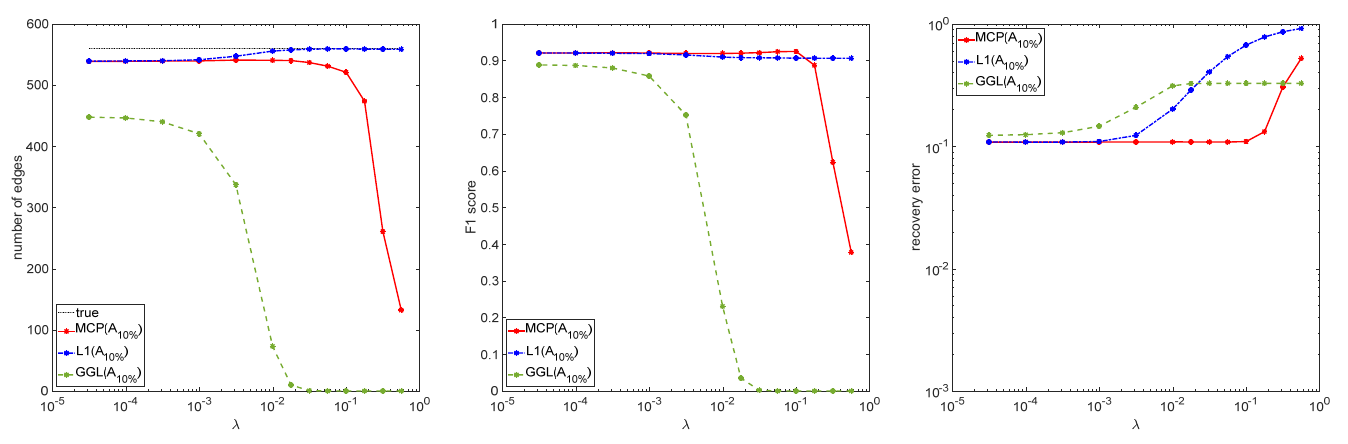}
  \caption{Results for random modular graph, $\mathcal{G}_M^{(100,0.05,0.3)}$. Scenario:
  we have a rough estimation of the true sparsity pattern and $A = A_{\rm 10\%}$  is not exactly accurate.}\label{fig-mod4}
\end{figure}

\end{document}